\documentclass{article}

\PassOptionsToPackage{numbers}{natbib}
\usepackage[preprint]{neurips_2026}

\usepackage[utf8]{inputenc} 
\usepackage[T1]{fontenc}    
\usepackage[colorlinks=true,citecolor=purple,linkcolor=purple]{hyperref}
\usepackage{url}            
\usepackage{booktabs}       
\usepackage{amsfonts}       
\usepackage{nicefrac}       
\usepackage{microtype}      
\usepackage{xcolor}         

\usepackage{graphicx}
\usepackage{algpseudocode}
\usepackage{booktabs}
\usepackage{soul}
\usepackage{algorithm}
\usepackage{amsthm}
\usepackage{stmaryrd}
\usepackage{makecell}
\usepackage{amssymb} 
\usepackage{wrapfig}
\usepackage{multicol}

\mathcode`w="8000
\begingroup
\lccode`~=`w
\lowercase{\endgroup
  \def~{\textit{w}}}

\usepackage{amsmath, amsfonts, bm}

\usepackage{etoolbox}
\makeatletter
\patchcmd{\hyper@makecurrent}{%
    \ifx\Hy@param\Hy@chapterstring
        \let\Hy@param\Hy@chapapp
    \fi
}{%
    \iftoggle{inappendix}{
        \@checkappendixparam{chapter}%
        \@checkappendixparam{section}%
        \@checkappendixparam{subsection}%
        \@checkappendixparam{subsubsection}%
        \@checkappendixparam{paragraph}%
        \@checkappendixparam{subparagraph}%
    }{}%
}{}{\errmessage{failed to patch}}

\newcommand*{\@checkappendixparam}[1]{%
    \def\@checkappendixparamtmp{#1}%
    \ifx\Hy@param\@checkappendixparamtmp
        \let\Hy@param\Hy@appendixstring
    \fi
}
\makeatletter

\newtoggle{inappendix}
\togglefalse{inappendix}

\apptocmd{\appendix}{\toggletrue{inappendix}}{}{\errmessage{failed to patch}}

\def\eqref#1{equation~\ref{#1}}

\def\1{\bm{1}}

\def\va{{\bm{a}}}

\def\vg{{\bm{g}}}
\def\vh{{\bm{h}}}

\def\vo{{\bm{o}}}

\def\vz{{\bm{z}}}

\DeclareMathAlphabet{\mathsfit}{\encodingdefault}{\sfdefault}{m}{sl}
\SetMathAlphabet{\mathsfit}{bold}{\encodingdefault}{\sfdefault}{bx}{n}

\def\gA{{\mathcal{A}}}

\def\gN{{\mathcal{N}}}

\def\gR{{\mathcal{R}}}
\def\gS{{\mathcal{S}}}

\def\gZ{{\mathcal{Z}}}

\def\sR{{\mathbb{R}}}

\DeclareMathOperator*{\aggr}{\textsc{Aggr}}

\DeclareMathOperator*{\upd}{\textsc{Upd}}
\DeclareMathOperator*{\meanset}{\textsc{Mean}}

\newcommand{\hiergraph}{\mathcal{G}^{\bm{*}}}
\usepackage[acronym, nohypertypes={acronym}]{glossaries}

\newacronym{method}{HiMPo}{\textit{\underline{Hi}erarchical \underline{M}essage-Passing \underline{Po}licy}}
\newacronym{rl}{RL}{Reinforcement Learning}
\newacronym{marl}{MARL}{Multi-Agent Reinforcement Learning}
\newacronym{hrl}{HRL}{Hierarchical Reinforcement Learning}
\newacronym{frl}{FRL}{Feudal Reinforcement Learning}
\newacronym{fgrl}{FGRL}{Feudal Graph Reinforcement Learning}
\newacronym{posg}{POSG}{Partially-Observable Stochastic Game}
\newacronym{pomdp}{POMDP}{Partially-Observable Markov Decision Process}
\newacronym{ctde}{CTDE}{Centralized Training Decentralized Execution}
\newacronym{comm-madrl}{Comm-MADRL}{Communication Multi-Agent Deep RL}

\newacronym{himppo}{HiMPPO}{Hierarchical Message-Passing PPO}
\newacronym{gppo}{GPPO}{Graph Neural Network PPO}
\newacronym{ippo}{IPPO}{Independent PPO}
\newacronym{mappo}{MAPPO}{Multi-Agent PPO}
\newacronym{hgppo}{H-GPPO}{hierarchical GPPO}
\newacronym{hguppo}{H-GUPPO}{hierarchical Graph U-Nets PPO}

\newacronym{himppo-fl}{HiMPPO-FL}{Full Local HiMPPO}
\newacronym{himppo-nl}{HiMPPO-NL}{No Local HiMPPO}
\newacronym{himppo-er}{HiMPPO-ER}{Environment Reward HiMPPO}
\newacronym{himppo-2l}{HiMPPO-$2$L}{$2$-Levels HiMPPO}
\newacronym{himppo-sg}{HiMPPO-SG}{Static Graph HiMPPO}

\newacronym{gnn}{GNN}{Graph Neural Network}

\definecolor{hlcolor}{HTML}{80BCF2}
\definecolor{hlcolorlight}{HTML}{BDDCF9}
\definecolor{hlred}{HTML}{F69B93}
\definecolor{hlyellow}{HTML}{BFEA9F}
\DeclareRobustCommand{\hlcyan}[1]{{\sethlcolor{hlcolor}\hl{#1}}}
\DeclareRobustCommand{\hlcyanlight}[1]{{\sethlcolor{hlcolorlight}\hl{#1}}}
\DeclareRobustCommand{\hlred}[1]{{\sethlcolor{hlred}\hl{#1}}}
\DeclareRobustCommand{\hlyellow}[1]{{\sethlcolor{hlyellow}\hl{#1}}}

\theoremstyle{plain}
\newtheorem{theorem}{Theorem}[section]
\newtheorem{lemma}[theorem]{Lemma}
\newtheorem{corollary}[theorem]{Corollary}

\title{Hierarchical Message-Passing Policies for Multi-Agent Reinforcement Learning}

\author{
    Tommaso Marzi \textsuperscript{\rm 1},
    Cesare Alippi \textsuperscript{\rm 1,2},
    Andrea Cini \textsuperscript{\rm 3}\\[.3em]
    \textsuperscript{\rm 1} IDSIA USI-SUPSI, Università della Svizzera italiana, Lugano, Switzerland\\
    \textsuperscript{\rm 2} Politecnico di Milano, Milan, Italy\\
   \textsuperscript{\rm 3} EPFL, Lausanne, Switzerland
}

\begin{document}

\maketitle

\begin{abstract}
Decentralized \acrfull{marl} methods allow for learning scalable multi-agent policies, but suffer from partial observability and induced non-stationarity. 
These challenges can be addressed by introducing mechanisms that facilitate coordination and high-level planning.
Specifically, coordination and temporal abstraction can be achieved through communication~(e.g., message passing) and \acrfull{hrl} approaches to decision-making. 
However, optimization issues limit the applicability of hierarchical policies to multi-agent systems. As such, the combination of these approaches has not been fully explored.
To fill this void, we propose a novel and effective methodology for learning multi-agent hierarchies of message-passing policies. We adopt the feudal \acrshort{hrl} framework and rely on a hierarchical graph structure for planning and coordination among agents. 
Agents at lower levels in the hierarchy receive goals from the upper levels and exchange messages with neighboring agents at the same level. 
To learn hierarchical multi-agent policies, we design a novel reward-assignment method based on training the lower-level policies to maximize the advantage function associated with the upper levels. Results on relevant benchmarks show that our method performs favorably compared to the state of the art.
\end{abstract}

\section{Introduction}~\label{sec:introduction}
Designing information processing methods to support coordination and planning is a core challenge in \gls{marl}~\cite{littman1994markov, huh2023multi}. 
Purely centralized approaches~\cite{Gupta2017CooperativeMC} reduce the multi-agent problem to a single-agent formulation and thus cannot scale~\cite{yuan2023survey}. Conversely, decentralized methods~\cite{amato2024introduction} rely solely on local observations: while being scalable, operating at the level of individual agents can make the environment appear non-stationary as other agents update their policies~\cite{hernandez2017survey,papoudakis2019dealing}. The induced non-stationarity, combined with partial observability, can hinder coordination and high-level planning~\cite{jin2025comprehensive}. 
Similarly, \gls{ctde} approaches~\cite{foerster2016learning,yuan2023survey,amato2024introduction} train agents on all the available~(global) information, but are limited to relying exclusively on local observations at execution time~\cite{zhou2023centralized,lyu2023centralized}. 
To overcome these challenges, several methods allow for information sharing among agents. In this regard, the \gls{comm-madrl}~\cite{zhu2024survey} framework improves coordination and mitigates partial observability by leveraging agent communication~\cite{singh2018individualized, das2019tarmac,hu2024learning}, possibly during both training and execution~\cite{niu2021multi,nayak2023scalable}. 
\Gls{comm-madrl} approaches often employ graph-based representations~\cite{battaglia2018relational}, which allow for modeling interacting agents as \emph{connected} nodes within a graph structure~\cite{duan2024inferring,ijcai2024p434}. In graph-based \gls{marl}, local observations are usually processed by relying on weight-sharing message-passing~\glspl{gnn}~\cite{gilmer2017neural, bacciu2020gentle, bronstein2021geometric}. 
However, communication on \textit{flat} graphs~(i.e., graphs modeling only pairwise interactions) can introduce bottlenecks in information propagation~\cite{topping2022understanding,arnaiz2025oversmoothing}, hindering coordination and planning~\cite{kurin2021my,marzi2024feudal}. 
Conversely, \gls{hrl}~\cite{sutton1999between,makar2001hierarchical,barto2003recent,ghavamzadeh2006hierarchical} methods organize the learning system into a hierarchical decision-making structure. By doing so, \gls{hrl} introduces learning biases that facilitate spatiotemporal abstraction and long-term planning~\cite{kulkarni2016hierarchical,yang2020hierarchical}. 
In particular, \gls{frl}~\cite{dayan1992feudal} relies on hierarchies of policies where the upper levels send commands and rewards to lower levels.
Combining communication-based approaches with an \gls{hrl} framework such as \gls{frl} has the potential to bring together the best of both worlds, but it is not straightforward.
Indeed, \gls{frl} requires defining ad-hoc rewards for lower levels. This adds overhead for the designer and hinders end‑to‑end learning of the hierarchy of policies~\cite{vezhnevets2017feudal,marzi2024feudal}. 
Moreover, when a hierarchy is added on top of the base agents, the message-passing mechanisms must be adapted to account for the additional interacting entities. 

In this paper, we propose a novel \gls{hrl} method for communication and coordination among different policies in \gls{marl} problems. We do so by adopting the \gls{frl} framework and introducing a procedure to avoid the associated reward-shaping and optimization issues. Our approach allows for learning a multi-level hierarchy of feudal message-passing policies, thereby achieving high-level planning and coordination.
The hierarchical graph defining the feudal structure is determined according to the current state, which can change over time.
To train the resulting hierarchy of policies, we design a novel reward-assignment scheme in which lower-level policies are trained to maximize the advantage function associated with the upper levels. 
The resulting propagation mechanism is decentralized as each agent optimizes different reward signals that depend on the hierarchical structure at each time step. 
In particular, we theoretically show that the resulting objective remains aligned with maximizing agents' local reward. 
This also makes our method suitable for scenarios that are not fully cooperative~(provided a coherent definition of the hierarchical structure). 
Our hierarchical \gls{marl} approach enjoys the same benefits of decentralized communication methods~\cite{zhu2024survey,aloor2024towards}, while relying on the feudal paradigm to mitigate non-stationarity and partial observability. 
Policy optimization can be carried out by applying any policy gradient algorithm to local trajectories at different levels. 
To summarize, our \textit{contributions} are the following:
\begin{enumerate}
\addtolength\itemsep{-.7mm}
    \item We introduce a novel method based on \gls{frl} for learning multi-level hierarchies of message-passing policies in multi-agent systems.
    \item We propose a flexible and adaptive reward-assignment scheme that leverages hierarchical graph structures for learning multi-level feudal policies. 
    \item We provide theoretical guarantees that the proposed learning scheme generates level-specific reward signals aligned with the global task.
\end{enumerate}
Empirical results show that our method, named \gls{method}, achieves strong performance in challenging \gls{marl} benchmarks that require coordination among agents. \Gls{method} does not necessarily require a pre-defined graph structure to operate on, as it can simply rely on the hierarchy for communication. \Gls{method} is a step toward the design of effective \gls{marl} methods for coordination and high-level planning.

\section{Preliminaries}~\label{sec:preliminaries}

\paragraph{\acrlong{posg}} Multi-agent systems can be formalized as \glspl{posg}~\cite{hansen2004dynamic}, i.e., a tuple $\langle \gN, \gS, \{\gZ_i\}_{i\in\gN}, \{\gA_i\}_{i\in\gN}, \mathcal{P},\{\gR_i\}_{i\in\gN}, \gamma \rangle$ where $\gN$ is the set of agents, $\gS$ is a state space, $\gZ_i$ and $\gA_i$ are the observation and action spaces of the $i$-th agent, $\mathcal{P}: \gS \times \gA_1\times\dots\times\gA_\gN \rightarrow \gS$ is a Markovian transition function depending on the joint set of actions, $\gR_i: \gS \times \gA_1\times\dots\times\gA_\gN \rightarrow \sR$ is the reward function for the $i$-th agent, and $\gamma\in[0,1)$ is a discount factor. Local observations $\{\vz_t^i\}_{i\in\gN}$ are generated as a function of the global state $\textbf{S}_t$. We consider \textit{homogeneous} \glspl{posg}~\cite{terry2020revisiting, marl-book}, i.e., systems where agents have identical~(but individual) observation spaces, action spaces, and reward functions. Agents will then learn a policy with shared parameters~\cite{agarwal2020learning} that aims to maximize the sum of discounted local rewards. In fully cooperative tasks, agents share the same objective. In mixed (cooperative-competitive) settings, agents' decisions are also driven by self-interested motives, necessitating local rewards~\cite{marl-book}. 

\paragraph{Hierarchies of Feudal Policies} \Acrlong{frl}~\cite{dayan1992feudal} approaches organize the decision-making structure into a multi-level hierarchy of policies. In \gls{frl}, actions taken by lower levels are conditioned on the goals assigned by upper levels, which are also responsible for propagating reward signals to subordinate agents. 
In the \acrlong{fgrl}~\cite{marzi2024feudal} paradigm, feudal dependencies are represented with a hierarchical graph $\hiergraph$, each node corresponding to a level-specific entity. In particular, (sub-)managers are abstract nodes sending goals to their subordinates, while workers represent sub-components of the agent~(e.g., actuators). Each level maximizes a separate payoff signal: the top-level manager collects rewards directly from the environment, while lower levels receive an intrinsic reward based on the alignment with received commands. 

\section{Learning Multi-Agent Hierarchical Message-Passing Policies}~\label{sec:methodology} 

Existing architectures for learning hierarchies of feudal policies require designing ad-hoc reward mechanisms, which can result in optimization issues~\cite{vezhnevets2017feudal,marzi2024feudal}. 
Applying such methodologies directly to multi-agent systems would suffer from the same drawbacks. 
Furthermore, in \gls{marl}, training the collection of policies by relying on a shared~(global) reward neglects the individual contribution of each agent and assumes a fully cooperative scenario. 
To address these challenges, \gls{method} learns a feudal hierarchy of message-passing policies in a decentralized fashion by independently optimizing level-specific signals. 
In \gls{method}, the reward for each agent is defined based on the advantage function of the associated (sub-)manager at the upper level. Unlike previous works on multi-agent \gls{hrl}~\cite{xu2023haven, liu2023hierarchical}, agent-specific learning signals are generated by locally evaluating the contribution of each worker or (sub-)manager, rather than in a global fashion. 
Furthermore, in \gls{method}, hierarchical relationships can evolve over time.
The following section introduces the proposed methodology. 
As reference, we consider a $3$-level hierarchical graph~(refer to \autoref{fig:reward_assignment}, left). After summarizing the decision-making process, we illustrate how to construct the hierarchical graph $\hiergraph_t$ and use it to assign the level-specific rewards. Then, we provide details on the training routine and theoretical guarantees on the soundness of the proposed reward scheme.  We refer the reader to~\autoref{app:algorithm} for a schematic overview of the algorithm, in which we divide it into distinct phases that align with~\autoref{sec:himpo_dmp}.

\subsection{\gls{method} Decision-Making Process}~\label{sec:himpo_dmp} 
We now detail the decision-making process for a $3$-level implementation of \gls{method}.
At each step $t$, we represent the policies at different levels through a hierarchical graph $\hiergraph_t$. Each worker $w$ at the base level corresponds to an agent, i.e., $w\in\gN$; the set $\gN_\sigma$ of sub-managers $s$ (intermediate level) is derived from the hierarchy (see~\autoref{sec:build_hierarchy}); the top-level manager $m$ is a single node at the top of the hierarchy $\hiergraph_t$.
In \hlcyan{\textbf{Phase 1}}, information is processed within $\hiergraph_t$. We denote node-level representations at round $l$ of message passing as $\vh_t^{w,l}$~(workers), $\vh_t^{s,l}$~(sub-managers), and $\vh_t^{m,l}$~(top-level manager).
Processing functions can be implemented using, e.g., MLPs. 
Each worker (agent) $w$ initializes its representations $\vh_t^{w,0}$ by encoding the corresponding raw observation $\vz_t^w$. 
Then, the representation $\vh_t^{s,0}$ of each sub-manager $s\in\gN_\sigma$ is generated by processing and aggregating the representations of its subordinate workers; this workers' partition $\mathcal{C}_t^s$ is derived from $\hiergraph_t$. 
Similarly, the representation $\vh_t^{m,0} $ of the top-level manager $m$ is initialized based on sub-managers' representations. Since there is a single top-level manager $m$ at the highest level of the hierarchy, we set $\vh_t^{m} \doteq\vh_t^{m,0} $, as it has no neighboring nodes to propagate information with. Conversely, workers and sub-managers propagate information through their level-specific graph by performing $L_r$ message-passing rounds~\cite{gilmer2017neural} as 
\begin{equation}
    \vh^{i,l+1}_{t} = {\upd}^i_l\Big(\vh^{i,l}_{t}, \aggr_{j \in \mathcal{B}_t(i)}\Big\{\textsc{Msg}^i_l\big(\vh^{i,l}_{t}, \vh^{j,l}_{t}\big)\Big\}\Big),\label{eq:mp}
\end{equation}
where $\mathcal{B}_t(i)$ denotes the neighbors of the $i$-th node at time step $t$, $\upd^i_l$ and $\textsc{Msg}^i_l$ indicate learnable update and message function~(e.g., MLPs), and $\aggr$ is a permutation-invariant aggregation function~(e.g., the mean). In practice, update and message functions are shared among nodes belonging to the same level. Then, based on such representations, the top-level manager $m$ and each $s$-th sub-manager send goals $\vg^{m{\scriptscriptstyle\to} s}_t\in \sR^{d_\mu}$ and $\vg^{s{\scriptscriptstyle\to} w}_t\in \sR^{d_\sigma}$, respectively ({\hlyellow{\textbf{Phase 2}}}), and each $w$-th worker~(agent) takes an action $\va^{w}_t \in \gA$ ({\hlred{\textbf{Phase 3}}}) based on its local observation $\vo_{t}^{i}$ as 
\begin{align}
    \label{eq:pomdp}
    \begin{aligned}
    &\vg^{m\to s}_t = \pi_\mu(\vo_t^{m\to s}),&  &\vo_t^{m\to s} = \vh_t^{m} || \vh_t^{s,0} || \vh_t^{s,L_r}\quad &\text{(manager)}\\
    &\vg^{s\to w}_t = \pi_\sigma(\vo_t^{s\to w}),&  &\vo_t^{s\to w} = \vg^{m\to s}_t || \vh_t^{s,L_r} || \vh_t^{w,0} || \vh_t^{w,L_r}\quad &\text{(sub-manager)}\\
    &\va^{w}_t = \pi_\omega(\vo_t^{w}),&  &\vo_t^{w} =  \vg^{s\to w}_t || \vh_t^{w,0} || \vh_t^{w,L_r}\quad &\text{(worker)}
\end{aligned}
\end{align} 
where $\pi_\mu, \pi_\sigma, \pi_\omega$ indicate the policies of the manager, sub-managers, and workers, respectively. We remark that policies are shared across nodes at the same level, i.e., we learn a single level-specific policy for all the workers ($\pi_\omega$) and sub-managers ($\pi_\sigma$). Each policy effectively acts in a~(learned) latent \gls{pomdp}~\cite{kaelbling1998planning}, with observations and actions as in \autoref{eq:pomdp} and rewards defined in the next subsection. As such, for each node $i\in\hiergraph_t$ with level-specific policy $\pi_\square$, we can define trajectories $\tau_i$, value $V^{\pi_\square} (\vo_{t}^{i})$ and advantage $A^{\pi_\square} (\vo_{t}^{i}, \pi_\square(\vo_{t}^{i}))$ functions associated with observations and actions~(goals).
Upper levels operate at different time scales: we assume that the top-level manager and sub-managers send goals every $\alpha_\mu$ and $\alpha_\sigma$ environment step, respectively. We maintain temporal consistency within the hierarchy, i.e., $\alpha_\mu\geq\alpha_\sigma\geq 1$, so that high-level goals are propagated less frequently than low-level ones~\cite{dayan1992feudal}. To synchronize goal updates across levels, we set $\alpha_\mu = K\alpha_\sigma\doteq K\alpha$, with $K\in\mathbb{Z}^+$.
Similar to other works on \gls{hrl}~\cite{vezhnevets2017feudal,xu2023haven}, $K$ and $\alpha$ are hyperparameters that can be tuned according to the problem.

\begin{figure}[t]
\centering
\includegraphics[width=\textwidth]{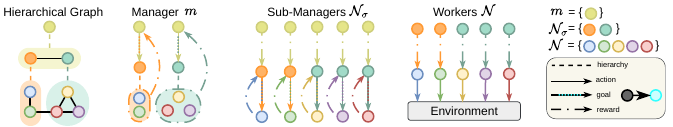}
\vspace{-0.45cm}
\caption{Assignment schemes of rewards and goals in a $3$-level hierarchical graph (up to temporal scales). Dashed lines indicate hierarchical relationships; solid arrows denote worker actions; solid arrows with colored dots indicate assigned goals; dash-dot arrows represent rewards.}
\label{fig:reward_assignment}
\end{figure}

\subsection{Building the Hierarchical Graph}\label{sec:build_hierarchy}
The graph $\hiergraph_t$ encoding the feudal hierarchy is built according to the environment structure. Following the Markovian assumption, it can be defined as a function of the global state $\textbf{S}_t$. 
In a $3$-level hierarchy, the set $\gN_\sigma$ of sub-managers can be dynamically learned~\cite{liu2025hygma} or defined based on prior knowledge~\cite{marzi2024feudal}. We refer the reader to~\autoref{app:sampling_hier_graph} for a visual example where we assign each sub-manager to a portion of the state space. All sub-managers are trivially connected to a single top-level manager, while the feudal relationships between workers and sub-managers can be derived from a subset of the observation features (e.g., agents' locations) or structural constraints (e.g., communication channels). 
We found agents' positions to provide a good enough bias to generate the hierarchical and relational structure without requiring additional knowledge on the environment~(see \autoref{sec:exp}).
For the upper levels of the hierarchy (sub-managers), we use fully-connected graphs to exchange messages, although more sophisticated solutions can be explored.
Note that the proposed methodology makes no assumptions about the graph used for message passing or the structure of the hierarchy.
When no prior knowledge exists, a viable approach is to connect all the workers directly to a single top-level manager, forming a static $2$-level hierarchy suitable for problems with no clear intermediate clustering~(see \autoref{sec:exp_lbfws}).
Hierarchical relationships can also be learned end-to-end~\cite{DBLP:journals/corr/abs-2208-05470,ijcai2024p434}, e.g., by relying on graph pooling operators~\cite{grattarola2022understanding}. Although more general, this approach introduces overheads in both computational and sample complexity. To keep the scope of the paper contained, we do not explore this direction.

\subsection{Hierarchical Assignment of Rewards}\label{sec:rewards}
In \gls{method}, (sub-)managers operate at different time scales, and the hierarchical graph $\mathcal{G}_t^*$ can be dynamic. Rewards must be assigned locally according to the topology of $\hiergraph_t$ to avoid sub-optimal solutions. If misspecified, rewards for intermediate levels can lead to undesirable behaviours, such as sub-managers competing against each other. \autoref{fig:reward_assignment} summarizes the reward-assignment scheme.

\paragraph{Manager} Each goal sent by the top-level manager is assigned a local learning signal, based on the rewards obtained by the subset of workers supervised by the sub-manager receiving the goal. 
In particular, the reward associated with the goal $\vg_{t}^{m{\scriptscriptstyle\to} s}$ sent by the manager $m$ to the $s$-th sub-manager at time step $t$ reads:
\begin{equation}
r_{t}^{m{\scriptscriptstyle\to} s} = \underset{w\in \mathcal{C}^{s}_t}{\meanset} \Biggl\{\sum_{k=0}^{K\alpha-1}\bar{r}_{t+k}^{w}\Biggr\},
\label{eq:reward_manager}
\end{equation}
where $\mathcal{C}^{s}_t$ denotes the partition of workers associated with sub-manager $s$ at step $t$ and $\{\bar{r}^{w}_{k}\}_{w=1}^N$ is the set of environment rewards at the $k$-th time step. In other words, $r_{t}^{m{\scriptscriptstyle\to} s}$ is defined as the average cumulative reward obtained by agents in $\mathcal{C}^{s}_t$ over the $K\alpha$ time steps associated with the goal $\vg^{m{\scriptscriptstyle\to} s}_t$. 
Notice that, for the duration of goal $\vg^{m{\scriptscriptstyle\to} s}_t$, changes in the subset of workers~(and topology $\hiergraph_t$) are not observed by the manager $m$. Thus, the cumulative reward $r_{t}^{m{\scriptscriptstyle\to} s}$ is computed by considering workers that are subordinate to $s$ at time $t$, i.e., the step when the goal was sent. We compute average (rather than total) cumulative rewards to account for variations in the number of subordinate workers w.r.t.\ future time steps. The reward in~\autoref{eq:reward_manager} allows us to define the manager’s advantage function in the latent \gls{pomdp}:
\begin{equation*}
    A^{\pi_\mu}(\vo_{t}^{m{\scriptscriptstyle\to} s}, \vg_{t}^{m{\scriptscriptstyle\to} s}) =\underset{\tau_m}{\mathbb{E}}\bigl[ r_{t}^{m{\scriptscriptstyle\to} s}+ \gamma V^{\pi_\mu} (\vo_{t+K\alpha}^{m{\scriptscriptstyle\to} s}) - V^{\pi_\mu} (\vo_{t}^{m{\scriptscriptstyle\to} s})\bigr]
\end{equation*}

\paragraph{Sub-managers}
Learning signals for lower levels must be aligned with the objectives set by the upper levels. We define the reward associated with each sub-manager $s$ according to the performance that manager $m$ achieved by assigning the goal $ \vg_{t}^{m{\scriptscriptstyle\to} s}$ to $s$. 
In particular, the reward signals are defined as the manager’s advantage function $A^{\pi_\mu} (\vo_{t}^{m{\scriptscriptstyle\to} s}, \vg_{t}^{m{\scriptscriptstyle\to} s})$, scaled by the number $K$ of sub-manager actions over the $K\alpha$ interval. Maximizing the manager's advantage function implies taking actions that lead to higher returns, encouraging sub-managers to conform to the set goals.
Note that $A^{\pi_\mu} (\vo_{t}^{m{\scriptscriptstyle\to} s}, \vg_{t}^{m{\scriptscriptstyle\to} s})$ is conditioned on the goal $ \vg_{t}^{m{\scriptscriptstyle\to} s}$ sent by the top-level manager. However, $\vg_{t}^{m{\scriptscriptstyle\to} s}$ is evaluated by aggregating cumulative rewards among workers in $\mathcal{C}_{t}^s$. In particular, from the perspective of the sub-manager $s$, $A^{\pi_\mu} (\vo_{t}^{m{\scriptscriptstyle\to} s}, \vg_{t}^{m{\scriptscriptstyle\to} s})$ is the same regardless of the $K$ number of times in which $s$ acts in the $K\alpha$ steps and the number $|\mathcal{C}_{t}^s|$ of goals sent at each of the step. 
Thus, we add a local component to the reward associated with each goal to distinguish it from goals assigned to other workers or at different time steps. Specifically, we use the cumulative reward received by worker $w$ over the subsequent $\alpha$ steps. As a result, the reward associated to $\vg_{t}^{s{\scriptscriptstyle\to} w}$ reads:
\begin{equation}
\label{eq:reward_submanager}
    r^{s{\scriptscriptstyle\to} w}_t = \frac{A^{\pi_\mu} (\vo_{t}^{m{\scriptscriptstyle\to} s}, \vg_{t}^{m{\scriptscriptstyle\to} s})}{{K}}  + \sum_{k = 0}^{{\alpha}-1}\bar{r}_{t+k}^w,
\end{equation}
and the corresponding sub-manager’s advantage function is:
\begin{equation*}
    A^{\pi_\sigma} (\vo_{t}^{s{\scriptscriptstyle\to} w}, \vg_{t}^{s{\scriptscriptstyle\to} w}) =\underset{\tau_s}{\mathbb{E}}\bigl[ r_{t}^{s{\scriptscriptstyle\to} w} + \gamma_\sigma V^{\pi_\sigma} (\vo_{t+\alpha}^{s{\scriptscriptstyle\to} w}) - V^{\pi_\sigma} (\vo_{t}^{s{\scriptscriptstyle\to} w})\bigr].
\end{equation*}

\paragraph{Workers}
We define rewards for workers similarly to sub-managers.
In particular, each worker $w$ receives a reward consisting of the advantage function $A^{\pi_\sigma} (\vo_{t}^{{s{\scriptscriptstyle\to} w}}, \vg_{t}^{s{\scriptscriptstyle\to} w})$ of the sub-manager $s$ from which the goal $\vg_{t}^{s{\scriptscriptstyle\to} w}$ has been sent, scaled by the duration ${\alpha}$, i.e., as
\begin{equation}
\label{eq:reward_worker}
    {r}_t^{w} = \frac{A^{\pi_\sigma} (\vo_{t}^{{s{\scriptscriptstyle\to} w}}, \vg_{t}^{s{\scriptscriptstyle\to} w})}{{\alpha}}, 
\end{equation}
where we denote the worker's reward as ${r}_t^{w}$ to distinguish it from the external reward $\bar{r}_t^{w}$. For each worker $w$, the reward in \autoref{eq:reward_worker} is distributed among the $\alpha$ actions taken in the interval $[t,t+\alpha)$. 
Note that such rewards are local, i.e., each worker receives a different learning signal. Unlike~\autoref{eq:reward_submanager}, there is no need for adding the external reward here, and, as a result, workers do not have direct access to the environment signals. For dynamic hierarchies, we found that using a modified advantage-like definition of $A^{\pi_\sigma} (\vo_{t}^{{s{\scriptscriptstyle\to} w}}, \vg_{t}^{s{\scriptscriptstyle\to} w})$ yields better results for policy optimization (see \autoref{app:ablation_truncation}).

\paragraph{End-to-end Training}
Following the \gls{frl} paradigm, each level-specific policy is trained independently from the others to maximize its level-specific return. This allows for generating goals at different spatial and temporal scales, enabling task decomposition within the hierarchy~(see~\autoref{sec:exp}). Although level-specific rewards depend on the advantage function at upper levels, ensuring they are aligned with each agent's objective is a significant challenge: we provide theoretical guarantees on this aspect in \autoref{sec:theory}.

\paragraph{Cooperative and Mixed Settings}~\label{par:applicability}  Note that, provided a proper choice of the hierarchy, \Gls{method} can be adopted in \textit{both} fully cooperative and self-interested settings: hierarchies with more than $2$-levels assume that the environment is cooperative, while a $2$-level hierarchy allows for mixed settings. Indeed, in a $3$-level hierarchy, the top-level manager aggregates workers' rewards from the partitions induced by the hierarchical structure $\hiergraph$~(refer to \autoref{eq:reward_manager}). 
Conversely, in $2$-level hierarchies, where workers are all connected to a single top-level manager, each worker's return is maximized separately. 
In this setting, each worker's learning signal is directly given by the manager's advantage function. Note that parameter-sharing can have an impact on the space of learnable multi-agent policies.

\subsection{Theoretical Analysis}\label{sec:theory}
\Gls{method} aims to learn a hierarchy of independent policies where each node receives a level-specific reward. 
To ensure the soundness of the proposed scheme, we must show that each level-specific objective aligns with the global task. 
Let $T_\mu = \{0, K\alpha, 2K\alpha,\dots,\infty\}$ denote the manager time scale, i.e., the time steps $t \in T_\mu$ at which goals $\vg_{t}^{m{\scriptscriptstyle\to} s}$ are sent. 
Given joint policy $\pi \doteq (\pi_\omega, \pi_\sigma, \pi_\mu)$, the global objective consists of maximizing the return of each worker $w$~(agent):
\begin{equation}
\label{eq:joint_return}
    \eta(\pi) =  \underset{\tau\sim \pi}{\mathbb{E}}\Biggl\lbrack \sum_{t\in T_\mu}\gamma^{t/K\alpha} \sum_{k=0}^{K\alpha - 1}\bar{r}_{t+k}^{w}\Biggr\rbrack,
\end{equation}
where $\tau$ is a trajectory and $\gamma\doteq\gamma_\mu$ is the discount factor of the top-level manager. 
The following theoretical results show that maximizing returns obtained from the rewards defined in \autoref{sec:rewards} implies maximizing the return in \autoref{eq:joint_return}. In particular, we assume that each level-specific policy is optimized while keeping the others fixed. Proofs for all the theoretical results are provided in \autoref{app:proofs}.

\begin{theorem}[Manager]
\label{thm:manager_optimization}
The maximization of the return $\eta_m(\pi_\mu)$ defined using the rewards in \autoref{eq:reward_manager} obtained by manager $m$ sending a goal to sub-manager $s$ implies the maximization of $ \eta(\pi)$, i.e.:
\begin{equation}
    \max_{\pi_\mu}\eta_m(\pi_\mu) \doteq  \max_{\pi_\mu}\underset{\tau_m\sim \pi}{\mathbb{E}}\Biggl\lbrack \sum_{{t}\in T_\mu}\gamma^{t/K\alpha} r_{t}^{m{\scriptscriptstyle\to} s} \Biggr\rbrack  =\max_{\pi_\mu}\eta(\pi)
\end{equation}
It follows that the manager's objective $\eta_m(\pi_\mu)$ is aligned with the global task in \autoref{eq:joint_return}. 
\end{theorem}

With a slight abuse of notation, let $\hat{\pi}_\omega$ and $\hat{\pi}_\sigma$ denote joint policies $({\pi}_\omega, \tilde{\pi}_\sigma, \tilde{\pi}_\mu) $ and $(\tilde{\pi}_\omega, {\pi}_\sigma, \tilde{\pi}_\mu)$, where $\tilde{\pi}_{i=\omega,\sigma,\mu}$ is the old version of $\pi_{i=\omega,\sigma,\mu}$. 
Furthermore, let $T_\sigma = \{0, \alpha, 2\alpha,\dots,\infty\}$ be the sub-managers' time scale, i.e., the set of time steps $t$ in which goals $\vg_t^{s{\scriptscriptstyle\to} w}$ are propagated.
\begin{theorem}[Sub-managers]
\label{thm:submanager_optimization}
Assuming $\gamma_\sigma\simeq\gamma\simeq1$ and $K$ not too large, the maximization of the return $\eta_s({\pi}_\sigma)$ defined using the rewards in \autoref{eq:reward_submanager} obtained by the sub-manager $s$ sending a goal to the worker $w$ implies the maximization of $ \eta(\hat{\pi}_\sigma)$, i.e.:
\begin{equation}
    \max_{\pi_\sigma}\eta_s({\pi}_\sigma) \doteq
     \max_{\pi_\sigma}    \underset{\tau_s\sim\hat{\pi}_\sigma}{\mathbb{E}}\Biggl[\sum_{t\in T_\sigma}\gamma_\sigma^{t/\alpha} r^{s{\scriptscriptstyle\to} w}_t\Biggr] =
     \max_{\pi_\sigma} \eta({\hat{\pi}_\sigma}) 
\end{equation}
It follows that the sub-manager's objective $\eta_s({\pi}_\sigma)$ is aligned with the global objective in \autoref{eq:joint_return}.
\end{theorem}

\begin{theorem}[Workers]
\label{thm:worker_optimization}
Assuming $\gamma_\omega\simeq\gamma_\sigma\simeq 1$ and $\alpha$ not too large, the maximization of the return $\eta_w({\pi}_\omega)$ defined using the rewards in \autoref{eq:reward_worker} obtained by the worker $w$ implies the maximization of $ \eta(\hat{\pi}_\omega)$, i.e.:
\begin{equation}
    \max_{\pi_\omega} \eta_w({\pi}_\omega) \doteq
     \max_{\pi_\omega} \underset{\tau_w\sim\hat{\pi}_\omega}{\mathbb{E}}\Biggl[\sum_{t = 0}^\infty\gamma^t_\omega r^{ w}_t\Biggr]  = \max_{\pi_\omega}\eta(\hat{\pi}_\omega)
\end{equation}
It follows that the worker's objective $\eta_w({\pi}_\omega)$ is aligned with the global objective in \autoref{eq:joint_return}.
\end{theorem}
Note that the theoretical results remain valid when using a $2$-level hierarchy. In such a case, the learning signal of each worker consists of the manager's advantage function. 
In practice, although the theoretical results assume that each policy is optimized while keeping the others fixed, we learn policies across different levels concurrently to improve sample efficiency; previous works empirically showed that this does not compromise performance~\cite{li2019hierarchical}. 

\section{Related Work}\label{sec:related}

Within communication-based methodologies~\cite{zhu2024survey}, message-passing architectures have been employed to parameterize policies~\cite{agarwal2020learning,niu2021multi} and enhance global awareness during training and execution without relying on global information~\cite{nayak2023scalable}. Other works follow a similar approach in the context of multi-robot systems~\cite{khan2020graph,blumenkamp2022framework,bettini2023heterogeneous}. \glspl{gnn} have also been used to learn action-value functions in different settings.
\citet{Jiang2020Graph} leverage an attention mechanism based on the agent graph to learn individual action-value functions, while other works~\cite{naderializadeh2020graph,kortvelesy2022qgnn} apply \glspl{gnn} to learn a global value function in fully cooperative scenarios.
Other \gls{comm-madrl} approaches either broadcast messages through static communication structures~\cite{foerster2016learning,sukhbaatar2016learning,das2019tarmac} or directly learn the propagation scheme~\cite{singh2018individualized,jiang2018learning,li2021deep,hu2024learning}. 
Communication has also been leveraged in several \gls{hrl} approaches~\cite{ryu2020multi, WANG2023359,liu2023deep, yang2023hierarchical}, but, to the best of our knowledge, none of these methods learn level-specific policies using independent, agent-specific reward signals. 
Concerning single-agent \gls{hrl} settings,~\citet{li2019hierarchical} addressed the issue of designing a problem-free reward function for low-level policies by leveraging the advantage function of high-level policies. Subsequently, this reward scheme has been extended to multi-agent systems~\cite{xu2023haven, liu2023hierarchical}, where architectures are trained under the \gls{ctde} paradigm using QMIX~\cite{rashid2020monotonic}. 
Unlike our approach, here low-level optimization signals are based on the team reward, i.e., actions are not explicitly evaluated according to their local contributions~\cite{10.5555/3237383.3238080}. Moreover, such methods are restricted to $2$-level hierarchies. 
We also note that value factorization~\cite{rashid2020monotonic} enforces global action optimality through specific architectural constraints.
This is fundamentally different from our theoretical contribution, which guarantees that the proposed advantage-based hierarchical rewards are aligned with the global task.
The \gls{frl}~\cite{dayan1992feudal} framework has been first extended to the deep RL setting~\cite{vezhnevets2017feudal} and, subsequently, applied to multi-agent systems~\cite{ahilan2019feudal}. In this context, it showed performance improvements across diverse domains, e.g., traffic signal control~\cite{ma2020feudal,ma2022feudal} or warehouse logistics~\cite{krnjaic2022scalable}. However, none of these methods adopts a problem-free reward scheme nor considers message-passing policies. 
Lastly, \gls{frl} has been exploited to learn hierarchical graph-based modular policies in single-agent RL~\cite{marzi2024feudal}.

\section{Experiments}\label{sec:exp}

In the following, we evaluate \gls{method} on relevant \gls{marl} benchmarks by considering problems of increasing complexity and with a varying number of interacting agents. We use Proximal Policy Optimization~(PPO)~\cite{DBLP:journals/corr/SchulmanWDRK17} as the underlying policy optimization algorithm and refer to this specific implementation of \gls{method} as \gls{himppo}.

\paragraph{Baselines and benchmarks}
We compare our methodology against state-of-the-art \gls{marl} methods. We consider $1)$ \gls{ippo} and $2)$ \gls{mappo}, i.e., two \gls{marl} variants of PPO proposed by~\citet{yu2022surprising}. Concerning graph-based methods, we include in the comparison $3)$ the homogeneous version of \gls{gppo}~\cite{bettini2023heterogeneous}, $4)$ \acrshort{hgppo}, a hierarchical variant of GPPO akin to HAMA~\cite{ryu2020multi} that assumes an underlying fully-connected graph, and $5)$ \acrshort{hguppo}, a hierarchical PPO architecture based on Graph U-Nets~\cite{gao2019graph}. Furthermore, in~\autoref{app:app_comm_methods} we compare HiMPPO against two additional communication methods that are not based on PPO, i.e., $6)$ DGN~\cite{Jiang2020Graph} and $7)$ MAGIC~\cite{niu2021multi}. The main difference of our method with the hierarchical baselines is that the latter use a hierarchical GNN to learn a single weight-sharing policy, while \gls{himppo} leverages a (dynamic) hierarchical graph to learn level-specific policies and assign rewards.
We consider $2$ main environments under several different settings: $1)$ {\textit{Level-Based Foraging with Survival}~(LBFwS), i.e, }a novel version of the LBF~\cite{albrecht2013game,papoudakis2021benchmarking} environment, to show how the models perform in tasks with challenging temporal credit assignment; $2)$ the \textit{VMAS Sampling} environment from the VMAS simulator~\cite{bettini2022vmas}, to assess the exploration capabilities of the baselines. In \textit{LBFwS} we consider three levels of increasing difficulty, while for \textit{VMAS Sampling} we test three different team sizes. Additionally, in~\autoref{app:app_smac} we report results on $3)$ \textit{SMACv2}~\cite{ellis2023smacv2}, i.e., the improved version of the \textit{StarCraft Multi-Agent Challenge}~(\textit{SMAC})~\cite{samvelyan2019starcraft}. 
For the models using graph-based representations, we use a dynamic graph based on the location of the agents at each time step. In particular, two agents are connected if they are within their observation ranges. Due to the different computational requirements of the environments, for each configuration, we consider $8$ independent seeds in \textit{LBFwS}, $6$ in \textit{VMAS Sampling}, and $4$ in \textit{SMACv2}.
Additional details are reported in \autoref{app:exp_setting}.

\subsection{Level-Based Foraging with Survival~(LBFwS)}~\label{sec:exp_lbfws}
In \textit{LBFwS}, agents navigate in a grid world where resources of different value are randomly spawned~(\autoref{app:app_lbfws}). Consuming resources increases the individual reward, while delivering them to a fixed location extends the episode duration. Higher-value items need the cooperation of multiple agents to be obtained. An agent can learn either to simply maximize its reward~(individual strategy) or to sacrifice the local reward to extend the episode and possibly achieve a higher final return~(cooperative strategy). In particular, we consider a system of $10$ agents and three difficulty levels~(\emph{Easy}, \emph{Medium}, and \emph{Hard}) that correspond to progressively larger grid sizes and varying amounts of resources~(see \autoref{app:app_lbfws}). 
For \gls{himppo}, we consider a static $2$-level hierarchical graph with all the agents~(workers) connected to a single top-level manager that sends goals every $K=5$ steps. We investigate the impact of the goal scale $K$ in \autoref{app:app_goal_scale}.

\begin{figure*}[t]
\centering
\includegraphics[width=\textwidth]{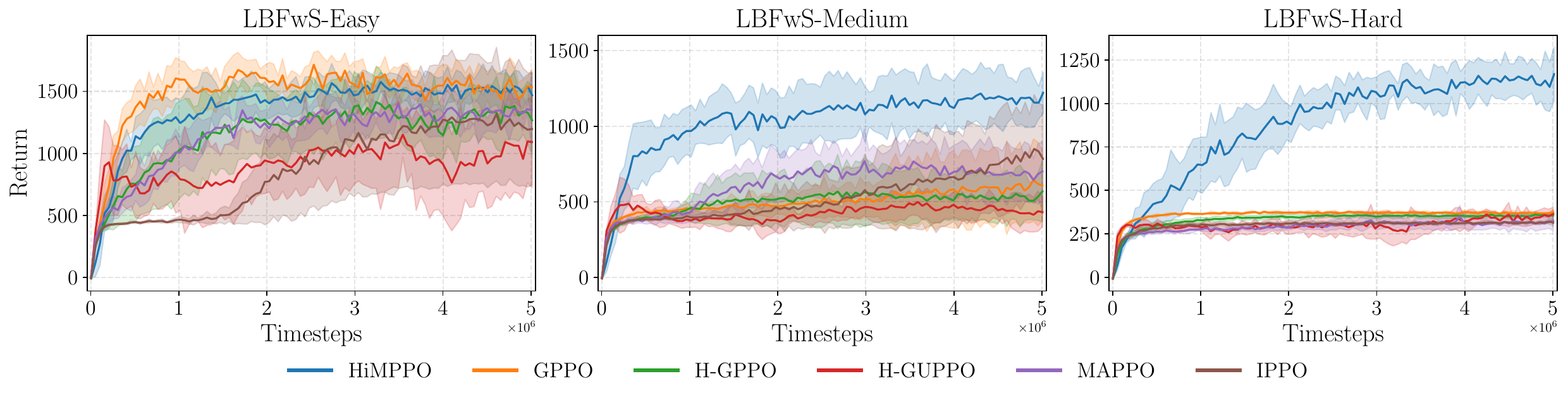}
\vspace{-0.7cm}
\caption{Results of the models for different difficulty levels in the \textit{LBFwS} environment~(avg. return $\pm$ std. of $8$ runs).}
\label{fig:lbf_results}
\end{figure*}

Results are reported in \autoref{fig:lbf_results}. 
All the models can achieve high returns in the \textit{LBFwS-Easy} scenario, with \gls{himppo} and GPPO being more sample efficient. However, as the difficulty increases, the baselines struggle to achieve good results. Notably, in \textit{LBFwS-Hard}, all the methods except \gls{himppo} learn to maximize the individual reward. Conversely, \gls{himppo} consistently achieves good results at every difficulty level. These results show the effectiveness of \gls{method} in coordination and planning capabilities. Furthermore, the performance achieved by other hierarchical methods (\acrshort{hgppo} and \acrshort{hguppo}) shows that relying on external reward fails in more complex scenarios, supporting the adoption of our reward-assignment mechanism.

\begin{figure*}[t]
\centering
\includegraphics[width=\textwidth]{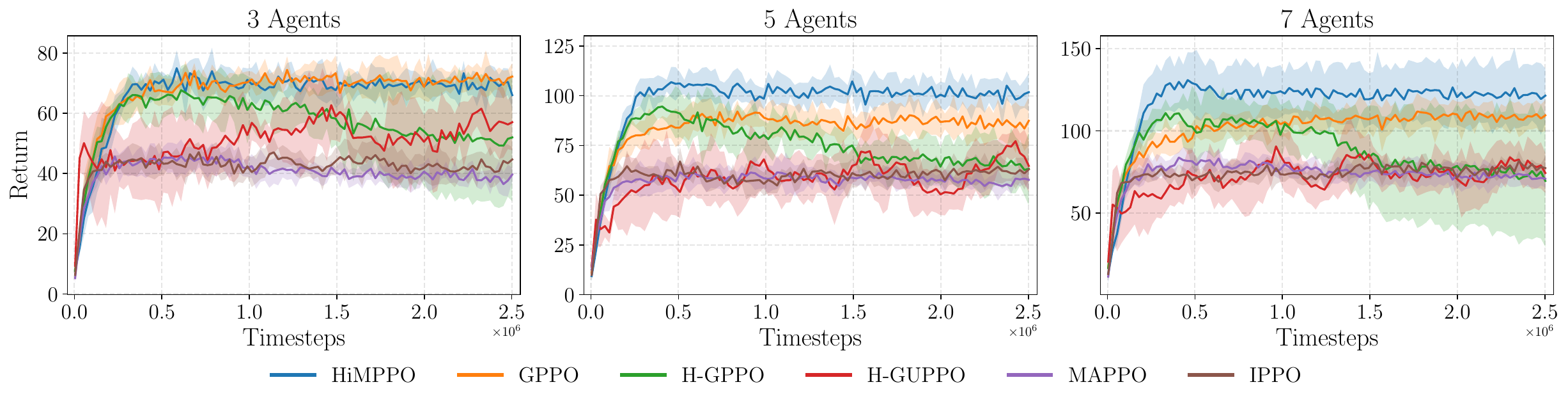}
\vspace{-0.7cm}
\caption{Results of the models for different numbers of agents in the \textit{VMAS Sampling} environment~(avg. return $\pm$ std. of $6$ runs).}
\label{fig:sampling_results}
\end{figure*}

\subsection{VMAS Sampling}~\label{sec:exp_vmas}
In this continuous control problem, $n$ robots are randomly spawned in a grid. The grid has an underlying Gaussian density function with $3$ modes. Visiting a cell for the first time leads to a reward proportional to the density function. 
For \gls{himppo}, the dynamic $3$-level hierarchical graph $\hiergraph_t$ is generated according to the positions of the robots at each step $t$~(refer to \autoref{app:sampling_hier_graph}); the goal scales are set to $K=2$ and $\alpha = 5$. 
In the advantage-like definition of the workers' rewards~(\autoref{app:proof_worker}), future signals are defined using a one-step truncation on the sequence of sub-managers' value functions~(refer to~\autoref{eq:dyn_adv_worker}), as empirical evidence showed better performance. \autoref{app:ablation_truncation} further analyzes the impact of considering different credit assignment schemes.

The results in \autoref{fig:sampling_results} show that \gls{himppo} achieves higher overall returns as the difficulty (i.e., the number of agents) increases, showing the benefits of relying on hierarchical policies to coordinate exploration in challenging scenarios. Conversely, we observed high training instability in the other hierarchical baselines (\acrshort{hgppo} and \acrshort{hguppo}), further confirming the effectiveness of our novel reward-assignment mechanism.

\subsection{Ablation Studies}

In this section, we perform a set of ablation studies to assess the impact and effectiveness of the proposed designs. 
Additional experiments and analyses are reported in~\autoref{app:app_other_experiments}.

\paragraph{Reward Assignment} 
To further validate the effectiveness of the hierarchical reward-assignment scheme proposed in \autoref{sec:rewards}, we compare \gls{himppo} against different variants in both \textit{LBFwS} and \textit{VMAS Sampling}. 
In particular, we consider $1)$ \gls{himppo-fl}, where we expose workers to the external reward by adding it to \autoref{eq:reward_worker}, and $2)$ \gls{himppo-er}, where workers' rewards are based only on the external signal~(the top-level manager still receives the reward in \autoref{eq:reward_manager}). Additionally, since in \textit{VMAS Sampling} we employ a $3$-level hierarchy, in this environment we also consider $3)$ \gls{himppo-nl}, where sub-managers do not receive the cumulative reward associated with each worker (refer to~\autoref{eq:reward_submanager}).
The results reported in \autoref{tab:tab_reward_assignment} show that our advantage-based hierarchical reward-assignment scheme is crucial for achieving good results in complex scenarios of \textit{LBFwS}, further confirming what we already observed when comparing our approach against hierarchical baselines. Results in the \textit{VMAS Sampling} environment show that preventing workers from accessing the external signal and employing advantage-like rewards do not deteriorate performance and improve sample efficiency (refer to \autoref{fig:sampling_ablation_reward}). Conversely, preventing sub-managers from using local rewards leads to poor results, as shown by \gls{himppo-nl}.

\def\thickhline{\noalign{\hrule height1.3pt}}

\begin{table}
\caption{Sensitivity to different reward-assignment mechanisms in \textit{LBFwS} and \textit{VMAS Sampling}~(avg. return $\pm$ std. of $8$ and $6$ runs, respectively). Training curves are reported in~\autoref{app:app_reward_mechanism}.}
\label{tab:tab_reward_assignment}
\setlength{\tabcolsep}{2.5pt}
\renewcommand{\arraystretch}{1.1}
\centering
\resizebox{0.8\columnwidth}{!}{
\begin{tabular}{c|c|c|c|c|c|c}
\thickhline
\multicolumn{1}{c|}{} & \multicolumn{3}{c|}{\textit{LBFwS}} & \multicolumn{3}{c}{\textit{VMAS Sampling}}\\
\multicolumn{1}{c|}{}  & \multicolumn{1}{c}{\textit{Easy}}& \multicolumn{1}{c}{\textit{Medium}}& \textit{Hard} & \multicolumn{1}{c}{3 Agents} & \multicolumn{1}{c}{5 Agents}& {7 Agents}\\
\thickhline
\acrshort{himppo} &  $1467\pm \scriptstyle{183}$ & $\boldsymbol{1221}\pm \scriptstyle{138}$& $\boldsymbol{1168}\pm \scriptstyle{155}$& $66\pm \scriptstyle{4}$& $\boldsymbol{102}\pm \scriptstyle{9}$& $\boldsymbol{121} \pm \scriptstyle{16}$ \\
 \acrshort{himppo-fl}&  $\boldsymbol{1682}\pm \scriptstyle{290}$& $768\pm \scriptstyle{435}$ & $381\pm \scriptstyle{12}$ &$68\pm \scriptstyle{4}$ & $91\pm \scriptstyle{9}$& $\boldsymbol{117}\pm \scriptstyle{12}$ \\
\acrshort{himppo-er} & $565\pm \scriptstyle{217}$ & $446\pm \scriptstyle{9}$& $375\pm \scriptstyle{14}$& {$\boldsymbol{73}\pm \scriptstyle{4}$}& {$\boldsymbol{103}\pm \scriptstyle{7}$}& $\boldsymbol{117}\pm \scriptstyle{14}$\\
\acrshort{himppo-nl} & $\diagup$ & $\diagup$& $\diagup$& {$33\pm \scriptstyle{6}$}& {$52\pm \scriptstyle{4}$}& $67\pm \scriptstyle{7}$\\
\hline
\end{tabular}
}
\end{table}

\def\thickhline{\noalign{\hrule height1.3pt}}

\begin{wraptable}[11]{R}{0.5\textwidth}
\vspace{-0.3cm}
\caption{Ablation study for different hierarchical graphs in the \textit{VMAS Sampling} environment~(avg. return $\pm$ std. of $6$ runs). Training curves are reported in~\autoref{app:app_hier_structure}. 
}
\vspace{0.2cm}
\label{tab:tab_hierarchy}
\renewcommand{\arraystretch}{1.1}
\centering
\resizebox{0.5\columnwidth}{!}{
\begin{tabular}{c|c|c|c}
\thickhline
\multicolumn{1}{c|}{}  & $3$ Agents& $5$ Agents& $7$ Agents \\
\thickhline
\acrshort{himppo} &  $\boldsymbol{66}\pm \scriptstyle{4}$ & $\boldsymbol{102}\pm \scriptstyle{9}$& $\boldsymbol{121}\pm \scriptstyle{16}$\\
 \acrshort{himppo-sg}&  ${51}\pm \scriptstyle{8}$ & ${80}\pm \scriptstyle{9}$& ${82}\pm \scriptstyle{9}$ \\
\acrshort{himppo-2l} &  $\boldsymbol{59}\pm \scriptstyle{9}$ & ${78}\pm \scriptstyle{14}$& ${96}\pm \scriptstyle{7}$\\
\hline
\end{tabular}
}
\end{wraptable}
\paragraph{Hierarchical Structure} We consider different hierarchical graph structures to investigate their impact on decision making. In particular, we compare the state-dependent dynamic hierarchical graph $\hiergraph_t$ (refer to \autoref{sec:build_hierarchy} and \autoref{app:sampling_hier_graph}) against: $1)$ \gls{himppo-sg}, a $3$-level \gls{himppo} where the topology is fixed at the beginning of each episode, i.e., the $\hiergraph_t = \hiergraph_0\ \forall\ t$; $2)$ \gls{himppo-2l}, where we use a static $2$-level hierarchy with all the workers connected to a single top-level manager.
The results reported in \autoref{tab:tab_hierarchy} show that the dynamic $3$-level hierarchy is crucial for achieving better performance in \textit{VMAS Sampling}. Indeed, \gls{himppo-sg} and \gls{himppo-2l} are limited compared to \gls{himppo} since the hierarchical structure does not reflect the current state. Notably, \gls{himppo-sg} and \gls{himppo-2l} perform akin. This suggests that, in this environment, processing information with a deeper (but static) hierarchy does not provide additional advantages when compared to a $2$-level structure.

\section{Conclusions}~\label{sec:conclusion}
We introduced~\acrfull{method}, a novel methodology for learning hierarchies of message-passing policies in multi-agent systems. The decision-making structure of \gls{method} relies on a (dynamic) multi-level hierarchical graph, which allows for improved coordination and planning by leveraging communication and hierarchical decision-making. Notably, the proposed reward-assignment method is theoretically sound, and empirical results on different benchmarks validate its effectiveness. \gls{method} opens up new designs for graph-based hierarchical \gls{marl} policies.

\paragraph{Limitations and future work}
The current implementation of \gls{method} uses a goal-assignment mechanism where upper levels operate at fixed time scales.  
Future work can build upon \gls{method} to explore more advanced goal-assignment schemes. In particular, it would be interesting to develop asynchronous mechanisms, where each (sub-)manager can send commands at different time scales based on the current state. 
Finally, extensions to heterogeneous multi-agent systems would broaden the applicability of \gls{method}. 

\section{Acknowledgements}
This work was supported by the Swiss National Science Foundation grants No.~204061~(\emph{HORD GNN: Higher-Order Relations and Dynamics in Graph Neural Networks}) and No.~225351~(\emph{Relational Deep Learning for Reliable Time Series Forecasting at Scale}).

\clearpage

\bibliography{content}
\bibliographystyle{plainnat}

\appendix

\newpage
\appendix
\section*{Appendix}

\section{\gls{method}}~\label{app:algorithm}
\begin{algorithm}[h!]
\caption{\gls{method} Decision-Making Process (single step)}
\label{alg:alg_episode}
\begin{algorithmic}[1]
\Require Observations $\{\vz_t^w\}_{w\in\gN}$; graph $\hiergraph_t$.
\Ensure Set of actions $\{\va_t^w\}_{w\in\gN}$.
\vspace{0.25em}

\Procedure{MessagePassing}{$\vh^{\cdot, 0}, \hiergraph$}
    \For{$l \in \{1, \dots, L_r\}$}
        \State $\vh^{\cdot, l} \gets \text{MP}(\vh^{\cdot, l-1}, \hiergraph)$ \Comment{Using~\autoref{eq:mp}}
    \EndFor
    \State \Return $\vh^{\cdot, L_r}$
\EndProcedure
\vspace{0.25em}
    \Statex \# {\hlcyan{\textbf{Phase 1: Bottom-up Representation Propagation}}}
    \Statex \#\# {\hlcyanlight{\textbf{Phase 1a:} Workers}}
    \ForAll{worker $w \in \mathcal{N}$} \Comment{in parallel}
        \State $\vh_t^{w,0} \gets \text{MLP}_w(\vz_t^w)$
        \State $\vh_t^{w,L_r} \gets \Call{MessagePassing}{\vh_t^{w,0}, \hiergraph_t}$
    \EndFor
    \Statex \#\# {\hlcyanlight{\textbf{Phase 1b:} Sub-managers (every $\alpha$ steps)}}
    \If{$t \bmod \alpha = 0$}
        \ForAll{sub-manager $s\in\gN_s$} \Comment{in parallel}
            \State $\vh_t^{s,0} \gets \underset{w \in \mathcal{C}_t^s}{\aggr}\{\text{MLP}_s(\vh_t^{w,0})\}$
            \State $\vh_t^{s,L_r} \gets \Call{MessagePassing}{\vh_t^{s,0}, \hiergraph_t}$
        \EndFor
    \EndIf

    \Statex \#\# {\hlcyanlight{\textbf{Phase 1c:} Manager (every $K\alpha$ steps)}}
    \If{$t \bmod K\alpha = 0$}
        \State $\vh_t^{m} \gets \underset{s}{\aggr}\{\text{MLP}_m(\vh_t^{s,0})\}$
    \EndIf

\vspace{0.25em}
   \Statex \# \hlyellow{\textbf{Phase 2: Top-down Goal Generation}}
    \If{$t \bmod K\alpha = 0$} \Comment{Manager $\to$ sub-managers}
        \ForAll{sub-manager $s\in\gN_s$} \Comment{in parallel}
            \State $\vg_t^{m{\scriptscriptstyle\to} s} \gets \pi_\mu(\vh_t^{m} \mathbin\| \vh_t^{s,0} \mathbin\| \vh_t^{s,L_r})$
        \EndFor
    \Else
        \State $\vg_t^{m{\scriptscriptstyle\to} s} \gets \vg_{t-1}^{m{\scriptscriptstyle\to} s}$ for all $s\in\gN_s$
    \EndIf
    
    \If{$t \bmod \alpha = 0$} \Comment{Sub-managers $\to$ workers}
        \ForAll{worker $w$ (under sub-manager $s$)} \Comment{in parallel}
             \State $\vg_t^{s{\scriptscriptstyle\to} w} \gets \pi_\sigma(\vg_t^{m{\scriptscriptstyle\to} s} \mathbin\| \vh_t^{s,L_r} \mathbin\| \vh_t^{w,0} \mathbin\| \vh_t^{w,L_r})$
        \EndFor
    \Else
        \State $\vg_t^{s{\scriptscriptstyle\to} w} \gets \vg_{t-1}^{s{\scriptscriptstyle\to} w}$ for all $w\in\gN$
    \EndIf
\vspace{0.25em}

    \Statex \# \hlred{\textbf{Phase 3: Worker Action Selection}}
    \ForAll{worker $w \in \mathcal{N}$} \Comment{in parallel}
        \State $\va_t^w \gets \pi_\omega(\vg_t^{s{\scriptscriptstyle\to} w} \mathbin\| \vh_t^{w,0} \mathbin\| \vh_t^{w,L_r})$
    \EndFor

    \State \Return actions $\{\va_t^w\}_{w\in\gN}$.
\end{algorithmic}
\end{algorithm}

\clearpage

\section{Additional Proofs}\label{app:proofs}

The topology of the hierarchical graph $\hiergraph_t$ can change in time. In the execution phase, received states are aggregated according to the current topology, making nodes aware of the hierarchical decision-making structure. Nevertheless, long-term returns do not provide information concerning possible changes in $\hiergraph_t$. As such, during the training phase, nodes cannot perceive changes in the topology using the sole learning signal. For example, consider the top-level manager $m$: at each step $t\in T_\mu $, goals $\vg_{t}^{m{\scriptscriptstyle\to} s}$ are sent according to $\hiergraph_t$, which accounts for differences in the partitions of the underneath nodes. However, in the optimization phase, such differences do not emerge when considering the long-term return. To address this problem, during training, we assume a static hierarchy given by the topology at the first step. In other words, each trajectory is evaluated using the hierarchical graph $\hiergraph_0$ at the first step $t=0$. 

\paragraph{Homogeneity and weight-sharing policies} As reported in~\autoref{sec:preliminaries}, we consider typical homogeneous settings~\cite{terry2020revisiting,marl-book}, i.e., multi-agent systems where the agents' individual dynamics are identical. This implies that the observation spaces, action spaces, and the reward function are the same. Consequently, they will also have the same expected return. For this reason, we learn a single weight-sharing policy $\pi_\omega$ that, for each agent, generates local~(individual) actions based on local observations. Similarly, sub-managers are also homogeneous, and we denote the corresponding weight-sharing policy as $\pi_\sigma$. Therefore, the theoretical results can be derived for a generic worker $w$ and sub-manager $s$. Note that the top-level manager $m$ is a single agent at the highest level of the hierarchy, and it has no neighboring agents. 

\subsection{Proof of \autoref{thm:manager_optimization}}~\label{sec:proof_manager}
\begin{proof}
Following the Markovian assumption, the probability of having a subset $\mathcal{C}_t^s$ of workers subordinate to $s$ at time $t$ depends only on the global state $\textbf{S}_t$ (refer to~\autoref{sec:build_hierarchy} and~\autoref{app:sampling_hier_graph}). 
The top-level manager takes actions according to the current topology $\hiergraph_t$. Therefore, long-term returns achieved by sending a goal $\vg_{t}^{m{\scriptscriptstyle\to} s}$ must be evaluated by considering the rewards associated with the subset of workers $\mathcal{C}_s^t$, i.e., the objective of the manager is defined by aggregating rewards w.r.t.\ the initial partition of the workers. This is a reasonable assumption since the scalar learning signal~(return) received by the manager does not provide any information concerning possible changes in the topology in the subsequent steps: the manager learns to provide optimal actions according to the current topology $\hiergraph_t$, and returns must be coherently defined under this setting. 
Given the joint policy $\pi \doteq (\pi_\omega, \pi_\sigma, \pi_\mu)$, the manager's objective reads:
\begin{multline}
    \eta_m(\pi_\mu) =  \underset{\textbf{S}_0}{\mathbb{E}}\underset{\vo_0^{m{\scriptscriptstyle\to} s}}{\mathbb{E}}\Bigl\lbrack V^{\pi_\mu}(\vo_0^{m{\scriptscriptstyle\to} s}) \Bigr\rbrack
     = \underset{\tau_m\sim \pi}{\mathbb{E}}\Biggl\lbrack \sum_{{t}\in T_\mu}\gamma^{t/K\alpha} r_{t}^{m{\scriptscriptstyle\to} s} \Biggr\rbrack =  \\ 
    =\underset{\textbf{S}_0}{\mathbb{E}} \underset{\mathcal{C}_0^s}{\mathbb{E}}\underset{\tau_m\sim \pi}{\mathbb{E}}\Biggl\{ \sum_{{t}\in T_\mu}\gamma^{t/K\alpha} \underset{w\in \mathcal{C}^{s}_0}{\meanset} \Biggl\lbrack\sum_{k=0}^{K\alpha-1}\bar{r}_{t+k}^{w}\Biggr\rbrack\Biggr\} =  \\
    =
    \underset{\textbf{S}_0}{\mathbb{E}} \underset{\mathcal{C}_0^s}{\mathbb{E}}\Biggl\{\underset{w\in \mathcal{C}^{s}_0}{\meanset}\Biggl\lbrack \underset{\tau_m\sim \pi}{\mathbb{E}}\biggl\lbrack \sum_{{t}\in T_\mu}\gamma^{t/K\alpha} \sum_{k=0}^{K\alpha-1}\bar{r}_{t+k}^{w}\biggr\rbrack \Biggr\rbrack \Biggr\} =\\ 
     = \underset{\textbf{S}_0}{\mathbb{E}} \underset{\mathcal{C}_0^s}{\mathbb{E}} \Biggl\lbrack \underset{w\in \mathcal{C}^{s}_0}{\meanset}\Bigl\{\eta(\pi|\textbf{S}_0)\Bigr\}\Biggr\rbrack =
     \underset{\textbf{S}_0}{\mathbb{E}} \Biggl\lbrack \eta(\pi|\textbf{S}_0)\Biggr\rbrack =
     \eta(\pi)
\end{multline}
In the second-to-last equality, we used the homogeneity of the agents~\cite{marl-book}, which implies that the objective does not depend on any specific worker $w$ as all workers will undergo the same dynamics and have the same expected return~(see also \autoref{sec:preliminaries}). Therefore, the manager's objective is aligned with the optimization goal of the joint policy~(refer to \autoref{eq:joint_return}). This proves \autoref{thm:manager_optimization}.
\end{proof}

\subsection{Intermediate Results}

In order to provide guarantees of alignment for the lower levels, we now report two intermediate results.
To show that low-level returns are aligned with the global objective, let $\tilde{\pi}$ and $\tilde{\pi}_{i=\omega,\sigma,\mu}$ denote the old versions of $\pi$ and $\pi_i$, respectively. Consider the following auxiliary term:
\begin{equation}
    \eta_{adv}({\pi}) \doteq \underset{\tau_m\sim{\pi}}{\mathbb{E}}\Biggl[\sum_{t\in T_\mu}\gamma^{t/K\alpha}  A^{\tilde{\pi}_\mu} (\vo_{t}^{m{\scriptscriptstyle\to} s}, \vg_{t}^{m{\scriptscriptstyle\to} s})\Biggr]
\end{equation}

\begin{lemma}
\label{lemma:adv_over_old_policy}
The objective $\eta({{\pi}})$ can be written in terms of the advantage over the old version $\tilde{\pi}$ of the joint policy as:
\begin{equation}
\label{eq:lemma}
    \eta({\pi}) =  \eta({\tilde{\pi}}) + \eta_{adv}({\pi})
\end{equation}
\end{lemma}

\begin{proof}
We can write the expected return $\eta(\pi)$ of the policy $\pi$ in terms of its expected advantage over the old version $\tilde{\pi}$:
{\allowdisplaybreaks
    \begin{multline}
     \eta_{adv}({\pi}) =
     \underset{\tau_m\sim{\pi}}{\mathbb{E}}\Biggl[\sum_{t\in T_\mu}\gamma^{t/K\alpha}  A^{\tilde{\pi}_\mu} (\vo_{t}^{m{\scriptscriptstyle\to} s}, \vg_{t}^{m{\scriptscriptstyle\to} s})\Biggr] =\\
   = \underset{\tau_m\sim{\pi}}{\mathbb{E}}\Biggl[\sum_{t\in T_\mu}\gamma^{t/K\alpha} \bigl[r_{t}^{m{\scriptscriptstyle\to} s} + \gamma V^{\tilde\pi_\mu} (\vo_{t+K\alpha}^{m{\scriptscriptstyle\to} s}) - V^{\tilde\pi_\mu} (\vo_{t}^{m{\scriptscriptstyle\to} s}) \bigl]\Biggr] =\\
   =\underset{\tau_m\sim{\pi}}{\mathbb{E}}\Biggl[\sum_{t\in T_\mu}\gamma^{t/K\alpha}r_{t}^{m{\scriptscriptstyle\to} s} -  V^{\tilde\pi_\mu} (\vo_{0}^{m{\scriptscriptstyle\to} s})\Biggr]
  =  \underset{\tau_m\sim{\pi}}{\mathbb{E}}\sum_{t\in T_\mu}\gamma^{t/K\alpha} r_{t}^{m{\scriptscriptstyle\to} s}-  \underset{\textbf{S}_{0}}{\mathbb{E}} \underset{\vo_{0}^{m{\scriptscriptstyle\to} s}}{\mathbb{E}} V^{\tilde\pi_\mu} (\vo_{0}^{m{\scriptscriptstyle\to} s})  = \\
  = \Bigl[\eta_m({\pi}_\mu) -  \eta_m(\tilde{\pi}_\mu) \Bigr]
  = \Bigl[\eta({\pi}) - \eta(\tilde{\pi})\Bigr]
    \end{multline}}
where in the last equality we used \autoref{thm:manager_optimization}. Rearranging the terms, we obtain the equality in \autoref{eq:lemma}.
\end{proof}

\begin{corollary}
    \label{corollary:optimization}
    Given Lemma~\ref{lemma:adv_over_old_policy}, maximizing the global objective $\eta(\pi)$ can be written as:
    \begin{equation}
    \label{eq:optimization_corollary}
        \max_{{\pi}}\eta({\pi}) =  \max_{{\pi}}\eta_{adv}({\pi})
    \end{equation}
\end{corollary}

\begin{proof}
Since $\eta(\tilde{\pi})$ is independent of the joint policy $\pi$, it is straightforward to see that maximizing $\eta({\pi})$ is independent of $\eta(\tilde{\pi})$ and the optimization of $\pi$ can be expressed as \autoref{eq:optimization_corollary}.
\end{proof}

\subsection{Proof of~\autoref{thm:submanager_optimization}}\label{sec:proof_submanager}
\begin{proof}

As we did for the top-level manager, the return of the sub-manager $s$ sending goal $\vg_t^{s{\scriptscriptstyle\to} w}$ to the worker $w$ is computed by considering the sequence of rewards obtained by $w$. We remark that in the $K\alpha$ steps where the top-level manager is not sending goals, the advantage function of the manager in \autoref{eq:reward_submanager} is fixed for each of the $K$ goals sent by the sub-manager.
Therefore, given the joint policy $\hat{\pi}_\sigma \doteq (\tilde{\pi}_\omega, {\pi}_\sigma, \tilde{\pi}_\mu)$, the objective for the sub-manager $s$ reads:
{\allowdisplaybreaks
\begin{multline}
    \eta_s({{\pi}}_\sigma) = 
     \underset{\textbf{S}_0}{\mathbb{E}}\underset{\vo_0^{s{\scriptscriptstyle\to} w}}{\mathbb{E}}\Bigl\lbrack V^{{\pi}_\sigma}(\vo_0^{s{\scriptscriptstyle\to} w}) \Bigr\rbrack =
    \underset{\tau_s\sim\hat{\pi}_\sigma}{\mathbb{E}}\Biggl\{\sum_{t\in T_\sigma}\gamma_\sigma^{t/\alpha} r^{s{\scriptscriptstyle\to} w}_t\Biggr\}=\\
    =\underset{\tau\sim\hat{\pi}_\sigma}{\mathbb{E}}\Biggl\{\sum_{t\in T_\sigma}\gamma_\sigma^{t/\alpha} \biggr[\frac{1}{K} A^{\tilde{\pi}_\mu} (\vo_{t}^{m{\scriptscriptstyle\to} s}, \vg_{t}^{m{\scriptscriptstyle\to} s}) + \sum_{i = 0}^{{\alpha - 1}}\bar{r}_{t+i}^w\biggl]\Biggr\}  =\\
    =\underset{\tau_m\sim\hat{\pi}_\sigma}{\mathbb{E}}\Biggl\{\sum_{t\in T_\mu}\underset{\tau_s(t)\sim\hat{\pi}_\sigma}{\mathbb{E}}\Biggl[\sum_{k=0}^{K-1}\underset{\tau_w(k)\sim\hat{\pi}_\sigma}{\mathbb{E}}{\gamma_\sigma^{t/\alpha} \gamma_\sigma^{k} }\Biggl(\frac{1}{K}A^{\tilde{\pi}_\mu} (\vo_{t}^{m{\scriptscriptstyle\to} s}, \vg_{t}^{m{\scriptscriptstyle\to} s}) + \sum_{i = 0}^{{\alpha-1}}\bar{r}_{t+k\alpha+i}^w\Biggr)\Biggl]\Biggr\} \simeq\\
   \simeq \underset{\tau_m\sim\hat{\pi}_\sigma}{\mathbb{E}}\Biggl\{\sum_{t\in T_\mu}\frac{\gamma^{t/K\alpha} }{K}\frac{1-\gamma^K}{1-\gamma}A^{\tilde{\pi}_\mu} (\vo_{t}^{m{\scriptscriptstyle\to} s}, \vg_{t}^{m{\scriptscriptstyle\to} s})\Biggr\} + \underset{\tau_m\sim\hat{\pi}_\sigma}{\mathbb{E}}\Biggl[\sum_{t\in T_\mu}\Biggl({\gamma^{t/K\alpha} }\sum_{i = 0}^{{K\alpha-1}}\bar{r}_{t+i}^w\Biggr)\Biggr]=\\
   = \frac{1-\gamma^K}{K(1-\gamma)}\underset{\tau_m\sim\hat{\pi}_\sigma}{\mathbb{E}}\Biggl\{\sum_{t\in T_\mu}{\gamma^{t/K\alpha} }A^{\tilde{\pi}_\mu} (\vo_{t}^{m{\scriptscriptstyle\to} s}, \vg_{t}^{m{\scriptscriptstyle\to} s})\Biggr\} + \eta(\hat{\pi}_\sigma) = \\
   = \left(\frac{1-\gamma^K}{K(1-\gamma)} + {1}\right)\underset{\tau_m\sim\hat{\pi}_\sigma}{\mathbb{E}}\Biggl[\sum_{t\in T_\mu}\gamma^{t/K\alpha}  A^{\tilde{\pi}_\mu} (\vo_{t}^{m{\scriptscriptstyle\to} s}, \vg_{t}^{m{\scriptscriptstyle\to} s})\Biggr] + \eta(\tilde{\pi}) = k_\sigma\eta_{adv}({\hat{\pi}}_\sigma) + \eta(\tilde{\pi})
\end{multline}}
where used Lemma~\ref{lemma:adv_over_old_policy} and the approximation holds when $\gamma_\sigma\simeq\gamma\simeq 1$ and $K$ is not too large.
Following Corollary~\ref{corollary:optimization}, since $k_\sigma > 0 $, maximizing the objective of each sub-manager $\eta_s({\pi}_\sigma)$ implies maximizing the global objective $\eta(\hat{\pi}_\sigma)$. This proves \autoref{thm:submanager_optimization}.

\end{proof}

\subsection{Proof of \autoref{thm:worker_optimization}}~\label{app:proof_worker}
In general, goals received by workers are fixed for $\alpha$ steps. Therefore, the reward (i.e., the advantage function of the sub-manager) in \autoref{eq:reward_worker} is constant in the $\alpha$ steps where the worker acts. We provide separate proofs of \autoref{thm:worker_optimization} for static and dynamic hierarchies to ease the presentation, despite the final result being the same. As before, we introduce the joint policy $\hat{\pi}_\omega\doteq({\pi}_\omega, \tilde{\pi}_\sigma, \tilde{\pi}_\mu) $.

\begin{proof}[Proof - Static Hierarchies]
Let us denote as $s$ the supervisor of the worker $w$. The objective for $w$ reads:
\begin{multline}
     \eta_w({\pi}_\omega) =  
     \underset{\textbf{S}_0}{\mathbb{E}}\underset{\vo_0^{w}}{\mathbb{E}}\Bigl\lbrack V^{{\pi}_\omega}(\vo_0^{ w}) \Bigr\rbrack =
     \underset{\tau_w\sim\hat{\pi}_\omega}{\mathbb{E}}\Biggl\{\sum_{t = 0}^\infty\gamma^t_\omega r^{ w}_t\Biggr\} =
      \underset{\tau\sim\hat{\pi}_\omega}{\mathbb{E}}\Biggl\{\sum_{t = 0}^\infty\gamma^t_\omega \frac{A^{\tilde\pi_\sigma} (\vo_{t}^{{s{\scriptscriptstyle\to} w}}, \vg_{t}^{s{\scriptscriptstyle\to} w})}{{\alpha}}\Biggr\} = \\
      = \frac{1}{\alpha}\underset{\tau\sim\hat{\pi}_\omega}{\mathbb{E}}\Biggl\{\sum_{t\in T_\sigma}\gamma^t_\omega \sum_{i=0}^{\alpha - 1}\gamma_\omega^i\Bigl[r_{t}^{s{\scriptscriptstyle\to} w} + \gamma_\sigma V^{\tilde\pi_\sigma} (\vo_{t+\alpha}^{s{\scriptscriptstyle\to} w}) - V^{\tilde\pi_\sigma} (\vo_{t}^{s{\scriptscriptstyle\to} w}) \Bigr]\Biggr\} \simeq\\
    \simeq\frac{1- \gamma^\alpha}{\alpha(1-\gamma)}\underset{\tau\sim\hat{\pi}_\omega}{\mathbb{E}}\Biggl\{\sum_{t\in T_\sigma}\gamma^{t/\alpha}r_{t}^{s{\scriptscriptstyle\to} w} - V^{\tilde\pi_\sigma} (\vo_{0}^{s{\scriptscriptstyle\to} w})\Biggr\}=\\
=   \frac{1- \gamma^\alpha}{\alpha(1-\gamma)} \Biggl\{ \underset{\tau\sim\hat{\pi}_\omega}{\mathbb{E}}\Biggl[ \sum_{t\in T_\sigma}\gamma^{t/\alpha}r_{t}^{s{\scriptscriptstyle\to} w}\Biggr] - \underset{\textbf{S}_{0}}{\mathbb{E}} \underset{\vo_{0}^{s{\scriptscriptstyle\to} w}}{\mathbb{E}} \Biggl[V^{\tilde\pi_\sigma} (\vo_{0}^{s{\scriptscriptstyle\to} w})\Biggr]\Biggr\} =\\
= \frac{1- \gamma^\alpha}{\alpha(1-\gamma)}  \Bigl[k_\sigma\eta_{adv}(\hat{\pi}_{\omega}) + \eta(\tilde{\pi}) - \eta_s(\tilde{{\pi}}_\sigma)\Bigr]  
= k_\omega\Bigl[k_\sigma\eta_{adv}(\hat{\pi}_{\omega}) + \eta(\tilde{\pi}) - \eta_s(\tilde{{\pi}}_\sigma)\Bigr]  
\end{multline}
where we used \autoref{thm:submanager_optimization} and assumed $\gamma_\omega \simeq \gamma_\sigma \simeq 1$ and $\alpha$ not too large. Since $k_\omega > 0$ and given Corollary~\ref{corollary:optimization}, maximizing the objective $\eta_w({\pi}_\omega)$ of each worker implies maximizing  the global objective $\eta(\hat{\pi}_\omega)$. This proves \autoref{thm:worker_optimization} in the case of a static hierarchy.
\end{proof}

\begin{proof}[Proof - Dynamic Hierarchies]
As we did for (sub-)managers, we evaluate the trajectory by keeping the hierarchy fixed. In other words, we consider the sequence of rewards associated with the sub-manager $s_0$ at the first step $t=0$. Nevertheless, defining the future rewards when using a dynamic hierarchy is an ill-posed problem. Indeed, observations $\vo_{t}^{{s_0}{\scriptscriptstyle\to} w}$ at future steps $t>0$ are not properly defined when the actual supervisor $\delta_{t}$ of $w$ at time $t$ is not $s_0$. However, as reported at the end of \autoref{sec:rewards}, using a $3$-level hierarchy restricts the model's applicability to cooperative problems. We then expect sub-managers to be cooperative too, and we can make them benefit from the value functions of future supervisors $\delta_{t>0}$ of $w$. Therefore, fixed the initial sub-manager $s_0$ and given a future step $t$, we can define the advantage-like reward of $w$ under the supervision of the sub-manager $s_0$ as:
\begin{equation}
\label{eq:dyn_adv_submanager}
r^{w}_t =\frac{1}{\alpha}\Bigl[ r_{t}^{{s_0}{\scriptscriptstyle\to} w} +\gamma_\sigma V^{\pi_\sigma} (\vo_{t+\alpha}^{{\delta_{t+\alpha}}{\scriptscriptstyle\to} w}) \llbracket t +\alpha \leq t^*\rrbracket  - V^{\pi_\sigma} (\vo_{t}^{{\delta_{t}}{\scriptscriptstyle\to} w})\llbracket t \leq t^*\rrbracket  \Bigr]
\end{equation}
where the Iverson brackets $\llbracket\cdot\rrbracket$ allow for truncating the influence of the goal $\vg_0^{s_0{\scriptscriptstyle\to} w}$ to the sequence of future sub-managers after the truncation step $t^*$. 
The objective for $w$ reads:
\begin{multline}
     \eta_w({\pi}_\omega) =  
     \underset{\textbf{S}_0}{\mathbb{E}}\underset{\vo_0^{w}}{\mathbb{E}}\Bigl\lbrack V^{{\pi}_\omega}(\vo_0^{ w}) \Bigr\rbrack =
     \underset{\tau_w\sim\hat{\pi}_\omega}{\mathbb{E}}\Biggl\{\sum_{t = 0}^\infty\gamma^t_\omega r^{ w}_t\Biggr\} =\\
      = \frac{1}{\alpha}\underset{\tau\sim\hat{\pi}_\omega}{\mathbb{E}}\Biggl\{\sum_{t\in T_\sigma}\gamma^t_\omega \sum_{i=0}^{\alpha - 1}\gamma_\omega^i\Bigl[ r_{t}^{{s_0}{\scriptscriptstyle\to} w} +\gamma_\sigma V^{\tilde\pi_\sigma} (\vo_{t+\alpha}^{{\delta_{t+\alpha}}{\scriptscriptstyle\to} w}) \llbracket t +\alpha \leq t^*\rrbracket
      -V^{\tilde\pi_\sigma} (\vo_{t}^{{\delta_{t}}{\scriptscriptstyle\to} w})\llbracket t \leq t^*\rrbracket \Bigr]\Biggr\} \simeq\\
    \simeq\frac{1- \gamma^\alpha}{\alpha(1-\gamma)}\underset{\tau\sim\hat{\pi}_\omega}{\mathbb{E}}\Biggl\{\sum_{t\in T_\sigma}\gamma^{t/\alpha}r_{t}^{s_0{\scriptscriptstyle\to} w} - V^{\tilde\pi_\sigma} (\vo_{0}^{s_0{\scriptscriptstyle\to} w})\Biggr\} =\\
=   \frac{1- \gamma^\alpha}{\alpha(1-\gamma)} \Biggl\{ \underset{\tau\sim\hat{\pi}_\omega}{\mathbb{E}}\Biggl[ \sum_{t\in T_\sigma}\gamma^{t/\alpha}r_{t}^{s_0{\scriptscriptstyle\to} w}\Biggr] - \underset{\textbf{S}_{0}}{\mathbb{E}} \underset{\vo_{0}^{s_0{\scriptscriptstyle\to} w}}{\mathbb{E}} \Biggl[V^{\tilde\pi_\sigma} (\vo_{0}^{s_0{\scriptscriptstyle\to} w})\Biggr]\Biggr\} =\\
= \frac{1- \gamma^\alpha}{\alpha(1-\gamma)}  \Bigl[k_\sigma\eta_{adv}(\hat{\pi}_{\omega}) + \eta(\tilde{\pi}) - \eta_s(\tilde{{\pi}}_\sigma)\Bigr]  
= k_\omega\Bigl[k_\sigma\eta_{adv}(\hat{\pi}_{\omega}) + \eta(\tilde{\pi}) - \eta_s(\tilde{{\pi}}_\sigma)\Bigr]  
\end{multline}
where we used \autoref{thm:submanager_optimization} and assumed $\gamma_\omega \simeq \gamma_\sigma \simeq 1$ and $\alpha$ not too large. Since $k_\omega > 0$ and given Corollary~\ref{corollary:optimization}, maximizing the objective $\eta_w({\pi}_\omega)$ of each worker implies maximizing the global objective $\eta(\hat{{\pi}}_\omega)$. This proves \autoref{thm:worker_optimization} in the case of a dynamic hierarchy.
\end{proof}
Notice that a similar result for dynamic hierarchies can be achieved using an advantage-like reward defined in the time scale of the workers:
\begin{equation}
\label{eq:dyn_adv_worker}
r^{w}_t =\frac{1}{\alpha}\Bigl[ r_{t}^{{s_0}{\scriptscriptstyle\to} w} +\gamma_\sigma V^{\pi_\sigma} (\vo_{t+1}^{{\delta_{t+1}}{\scriptscriptstyle\to} w}) \llbracket t +1 \leq t^*\rrbracket  - V^{\pi_\sigma} (\vo_{t}^{{\delta_{t}}{\scriptscriptstyle\to} w})\llbracket t \leq t^*\rrbracket  \Bigr]
\end{equation}
where:
\begin{equation}
\vo_{t+1}^{{\delta_{t+1}}{\scriptscriptstyle\to} w} = 
    \begin{cases}
        \vo_{t+\alpha}^{{\delta_{t+\alpha}}{\scriptscriptstyle\to} w}\quad &\text{if}\  (t+1)\bmod\alpha = 0 \\
        \vo_{t}^{{\delta_{t}}{\scriptscriptstyle\to} w}\quad &\text{otherwise}
    \end{cases}
\end{equation}
By doing so, workers undergo the possible change of supervisor only at the last of the $\alpha$ steps. 
We conduct an empirical analysis of this aspect in \autoref{app:ablation_truncation}. 

\begin{figure*}[b!]
\centering
\includegraphics[width=\textwidth]{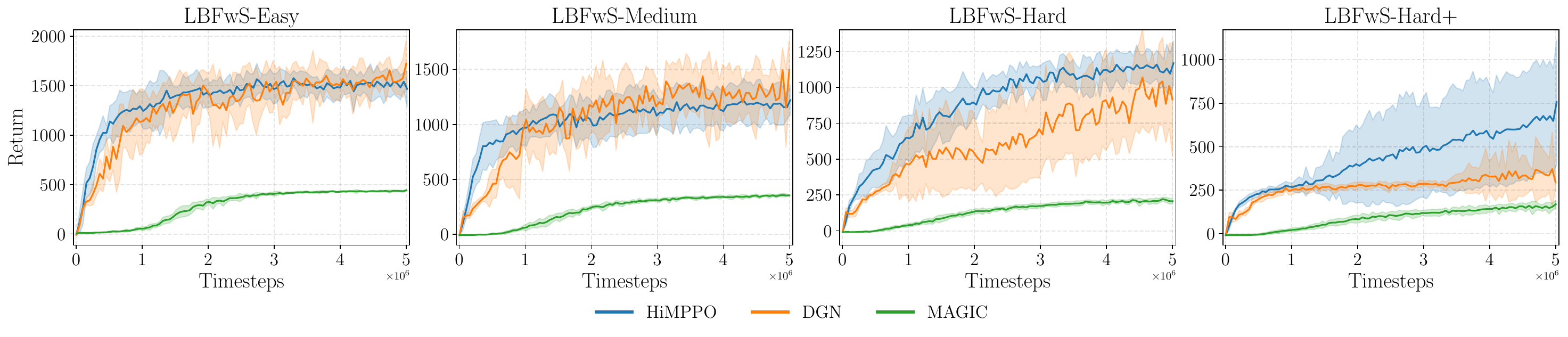}
\vspace{-0.7cm}
\caption{Results of \gls{himppo}, DGN, and MAGIC in the \textit{LBFwS} environment. For each configuration, we report the average return and sample standard deviation of $8$ runs. }
\label{fig:lbf_communication}
\end{figure*}

\section{Additional Results}~\label{app:app_other_experiments}
\subsection{Comparison with Other Communication Methods}~\label{app:app_comm_methods}
In this section, we compare \gls{himppo} against two additional communication methods that are not based on PPO. In particular, we consider $1)$ DGN~\citep{Jiang2020Graph}, which is based on DQN~\cite{mnih2013playing,mnih2015human} and relies on attention-based graph convolutions, and $2)$ MAGIC~\cite{niu2021multi}, a policy gradient method that constructs dynamic agent-agent relationships and then processes the messages using a graph attention network. We compare those models on \textit{LBFwS}, as it requires both coordination and high-level planning to be solved.

Results reported in~\autoref{fig:lbf_communication} show that MAGIC always learns the individual strategy, leading to lower returns. On the contrary, DGN performs comparably with \gls{himppo} in \textit{LBFwS-Easy} and \textit{LBFwS-Medium}, while in \textit{LBFwS-Hard} is still able to learn the cooperative strategy, even if it achieves lower average returns w.r.t. \gls{himppo} and is less sample efficient. Nevertheless, to further test the models' capabilities, we consider the \textit{LBFwS-Hard+} scenario, where the map is an $18\times18$ grid and there are $7$ food items of level $1$ and $7$ food items of level $2$. Here, DGN struggles to learn the cooperative strategy, while \gls{himppo} learns it in the majority of the runs.

\subsection{Analysis of the Communication Mechanism}~\label{app:app_sampling_topology}
To analyze the impact of the proposed coordination and communication mechanism, we compare the results achieved by \gls{himppo} against some variants of GPPO that use a fixed graph topology. In addition to H-GPPO, which assumes an underlying fully-connected graph, we consider a star graph~(GPPO-star), with a central node connected to all the others, a path graph~(GPPO-path), and a cycle graph~(GPPO-cycle). Note that for $3$-node graphs, path and cycle graphs are equivalent to star and fully-connected graphs, respectively. 

Results reported in \autoref{fig:sampling_topology} show that \gls{himppo} performs better than all the variants for larger systems. In particular, using fixed topologies in GPPO often hinders performance, and a fully-connected graph can make optimization unstable. Note that, globally, \gls{himppo} and fixed-topology variants of GPPO have access to the same information, showing that controlling how such information is processed is crucial in learning effective policies.

\begin{figure*}[t]
\centering
\includegraphics[width=\textwidth]{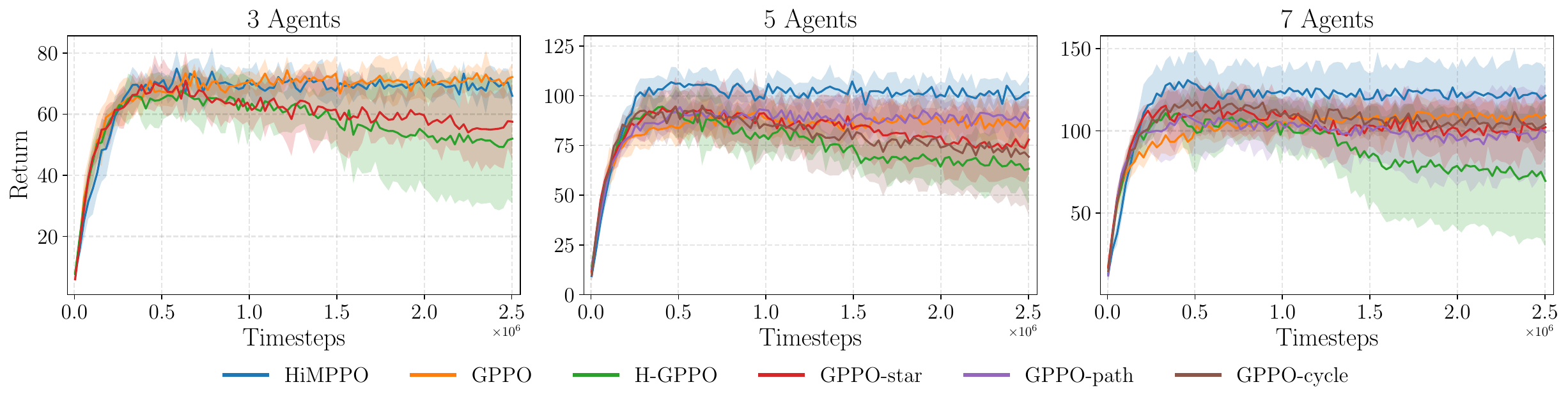}
\vspace{-0.7cm}
\caption{Results of the communication analysis in the \textit{VMAS Sampling} environment, where we compare the performance of \gls{himppo} against \gls{gppo} and its variants with different communication topologies. For each configuration, we report the average return and sample standard deviation of $6$ runs.}
\label{fig:sampling_topology}
\end{figure*}

\subsection{Sensitivity Analysis of the Goal Scale}~\label{app:app_goal_scale} 
The hyperparameters $K$ and $\alpha$ represent the time steps in which (sub-)managers act, i.e., the goal-propagation frequency. In general, these time scales align decision-making to the temporal resolution of the problem, and can be set by performing a grid search. Faster scales should be adopted in tasks requiring quick adaptation to provide more frequent supervision (e.g., robotic control). Conversely, slower scales can help in temporally extended tasks, such as navigation problems. 

To assess how the goal-propagation scale affects the performance of \gls{himppo}, we perform a sensitivity experiment of the manager's goal scale $K$ in \textit{LBFwS}. In particular, in addition to $K=5$, we consider slower ($K=15$) and faster ($K=2$) scales. The results in \autoref{fig:sensitivity_analysis_goal} show that, on average, our model is robust to the choice of $K$. Nevertheless, using a faster scale can effectively improve the performance in simpler tasks by providing more frequent commands, but can be detrimental in scenarios that require a coarser resolution. Indeed, because of the map size in \textit{LBFwS-Hard}, more steps are needed to find and deliver items in this scenario.

\begin{figure*}[t]
\centering
\includegraphics[width=1.0\textwidth]{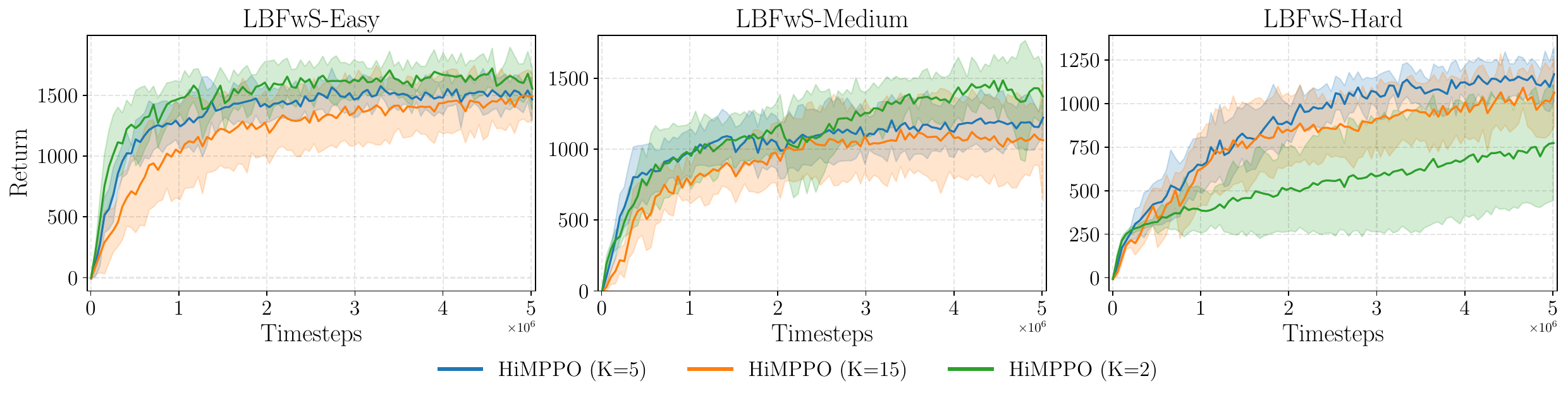}
\vspace{-0.7cm}
\caption{Results of the sensitivity analysis on the goal scale $K$ in the \textit{LBFwS} environment. For each configuration, we report the average return and sample standard deviation of $8$ runs.}
\label{fig:sensitivity_analysis_goal}
\end{figure*}

\begin{figure*}[t]
\centering
\includegraphics[width=1.0\textwidth]{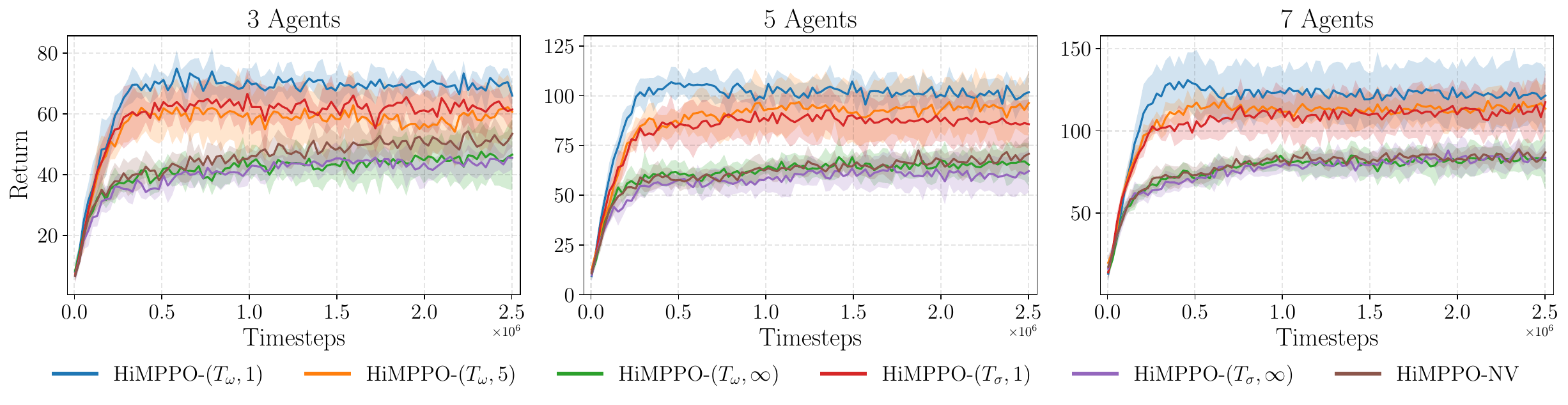}
\vspace{-0.7cm}
\caption{Results of the ablation study for different truncation schemes in the \textit{VMAS Sampling} environment. For each configuration, we report the average return and sample standard deviation of $6$ runs.}
\label{fig:ablation_truncation}
\end{figure*}

\subsection{Analysis of the Truncation Value}~\label{app:ablation_truncation} 

As reported in \autoref{app:proof_worker}, dynamic $3$-level hierarchies require defining future value functions when considering the sequence of rewards associated with the sub-manager at the first step. In particular, according to the desired temporal resolution, they can be defined w.r.t.\ the time scale of sub-managers~(\autoref{eq:dyn_adv_submanager}) or workers (\autoref{eq:dyn_adv_worker}).
Moreover, future values $V^{\pi_\sigma}(\vo_{t}^{{\delta_{t}}{\scriptscriptstyle\to} w})$ can be set to zero after a certain truncation step $t^*$. This threshold accounts for the degree of cooperation of the sub-managers: low values of $t^*$ imply that the sub-manager sending the goal relies only on first estimates of the value functions, and vice versa. 

To investigate these aspects, we consider both the value-assignment methods as well as different truncation values. We denote as \gls{himppo}-$(T_\sigma, t^*)$ and \gls{himppo}-$(T_\omega, t^*)$ the models where advantage-like rewards are defined using \autoref{eq:dyn_adv_submanager} and~\ref{eq:dyn_adv_worker}, respectively, with truncation step $t^*$. Furthermore, No Values \gls{himppo} (\gls{himppo}-NV) denotes the advantage-like rewards without value functions, i.e., with truncation step $t^*=0$. The results reported in \autoref{fig:ablation_truncation} show that for low $t^*$ the model performs better. However, completely truncating the value functions from the workers' reward~(\gls{himppo}-NV) leads to lower returns.
A possible explanation could be that the sub-managers' value functions are crucial for correctly guiding the workers toward the global objective. Nevertheless, high values of $t^*$ make credit assignment more challenging, as it considers a longer sequence of value functions, which might be related to other sub-managers.

\subsection{SMACv2}~\label{app:app_smac}

We consider $3$ representative maps from \textit{SMACv2}, namely \textit{Protoss}, \textit{Terran}, and \textit{Zerg}. Since hyperparameters were not tuned for this environment, we adopted the simplified $6\_vs\_5$ scenario for all the maps. For the models leveraging graphs, we defined two allies~(nodes) as connected if their distance is less than a pre-specified communication range, fixed to $3$; dead allies are filtered out from the graph.
For \gls{himppo}, the hierarchical scheme is a $2$-level hierarchy where goals are propagated every $5$ steps.

The results reported in \autoref{fig:smac_results} show that graph-based representations do not improve performance in this environment: allies' features are already encoded in the observation space of each agent, making IPPO and MAPPO the best-performing models. Indeed, independent learning can achieve remarkable performance in similar settings~\citep{de2020independent}, suggesting that additional information processing mechanisms might be redundant for solving these tasks. Nevertheless, the $2$-level hierarchical coordination scheme allows \gls{himppo} to achieve the best performance among graph-based methods.
Since \textit{SMACv2} is a fully cooperative environment, agents receive the same reward signal at each step. 
This makes it even more challenging for our model to learn an effective hierarchical coordination mechanism, as goals received by the workers at the same step are associated with the same long-term return. 
Therefore, the manager has to learn to generate individual goals based on an implicit estimate of the local contributions of each worker. Nevertheless, despite the complex reward-assignment mechanism, \gls{himppo} achieves a competitive win rate percentage.

\begin{figure*}[h!]
\centering
\includegraphics[width=\textwidth]{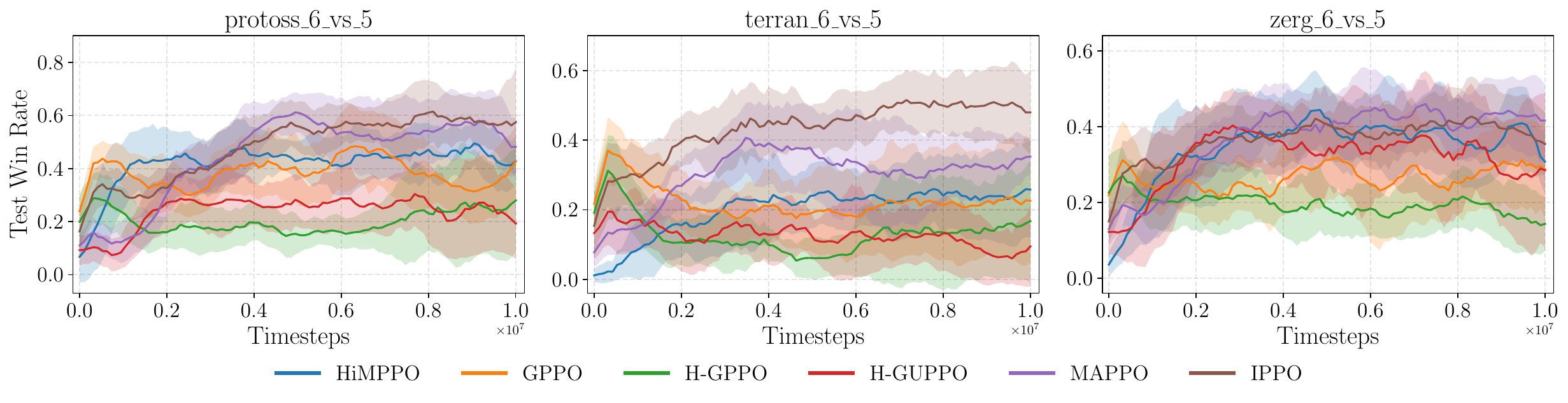}
\vspace{-0.7cm}
\caption{Results of the models for three maps of the \textit{SMACv2} environment. We report the average test win rate and sample standard deviation of $4$ runs for each configuration.}
\label{fig:smac_results}
\end{figure*}

\subsection{Analysis of the Reward Mechanism}~\label{app:app_reward_mechanism}

\begin{figure*}[h!]
\centering
\includegraphics[width=1.0\textwidth]{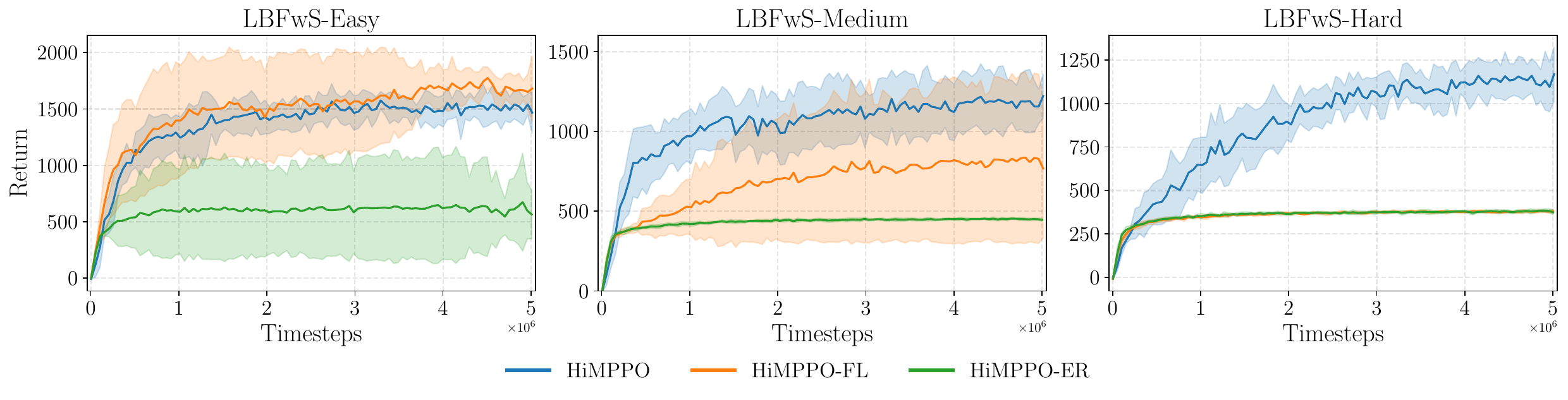}
\vspace{-0.7cm}
\caption{Results of the ablation study for different reward-assignment mechanisms in the \textit{LBFwS} environment. For each configuration, we report the average return and sample standard deviation of $8$ runs.}
\label{fig:lbf_ablation_reward}
\end{figure*}

\begin{figure*}[h!]
\centering
\includegraphics[width=1.0\textwidth]{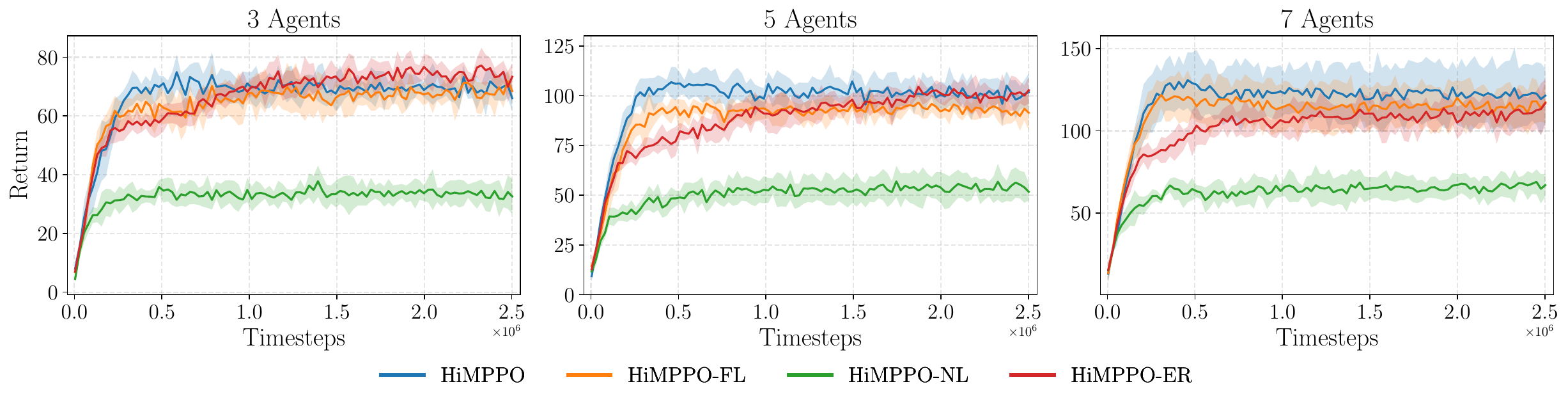}
\vspace{-0.7cm}
\caption{Results of the ablation study for different reward-assignment mechanisms in the \textit{VMAS Sampling} environment. For each configuration, we report the average return and sample standard deviation of $6$ runs.}
\label{fig:sampling_ablation_reward}
\end{figure*}

\clearpage
\subsection{Analysis of the Hierarchical Structure}~\label{app:app_hier_structure}

\begin{figure*}[h!]
\centering
\includegraphics[width=1.0\textwidth]{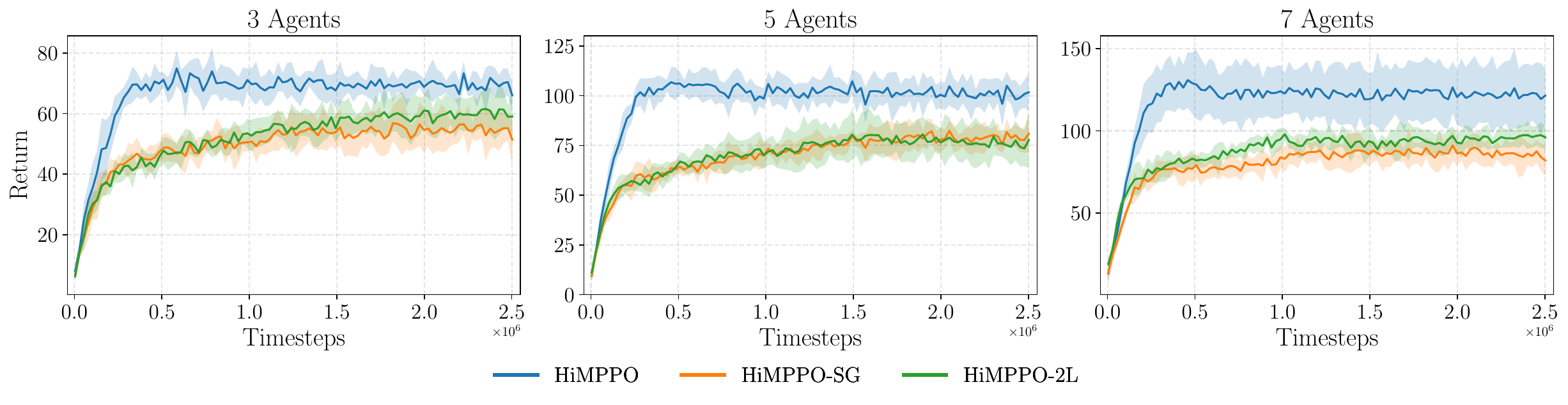}
\vspace{-0.7cm}
\caption{Results of the ablation study for different hierarchical graphs in the \textit{VMAS Sampling} environment. For each configuration, we report the average return and sample standard deviation of $6$ runs.}
\label{fig:ablation_hierarchy}
\end{figure*}

\section{Experimental Setting}~\label{app:exp_setting}
\subsection{Environments}~\label{app:environments}
\subsubsection{Level-Based Foraging with Survival~(LBFwS)}~\label{app:app_lbfws}
Starting from the codebase provided by~\citet{azmani2023cooperative}, we modify the Level-Based Foraging~\citep{albrecht2013game,papoudakis2021benchmarking} environment by further exacerbating the cooperative-competitive behavior of the agents to make it more challenging. In particular, once the food is within its range, an agent can eat it or deliver it to a landmark. Following the original version, the first strategy leads to an immediate, individual reward. Conversely, delivering food extends the episode duration: even if it gives no individual reward in the short term, it is beneficial for the team of agents, allowing them to have more time to collect food and achieve a potentially higher cumulative return. This behavior mirrors community food-sharing dynamics, where individuals contribute to the group's survival. This extended version of the environment, named \textit{Level-Based Foraging with Survival}~(LBFwS), aggravates the cooperative-competitive dynamics: in addition to evenly dividing the reward when cooperating~\citep{azmani2023cooperative}, trying to deliver food can be harmful to the agent, as other agents might act greedily to maximize their individual reward instead of adopting the group's survival strategy. Furthermore, delivering food requires a complex and temporally extended series of actions, i.e., collecting the item, transporting it to the landmark, and releasing it, making the problem challenging in terms of exploration and credit assignment.
We now illustrate the differences of the LBFwS environment with respect to the version of~\citet{azmani2023cooperative}.

\paragraph{Observation Space} Compared with the original setting, each agent has an additional grid representing its local observation of the landmark layer, i.e., if the landmark is within its sight range and if the agent (or its neighbors) is carrying food. Furthermore, each $i$-th agent observation is augmented by a $4$-dimensional vector $(t, t_{s}, x^i_{lm}, y^i_{lm})$, where $t$ is the time passed so far~(normalized with the absolute maximum number of steps $T$), $t_{s}$ is the survival time (which can be incremented by delivering food) normalized with the survival initial value $T_\sigma<T$, and $x^i_{lm}$ and $y^i_{lm}$ are the normalized relative displacements in the $(x,y)$ axis w.r.t. the landmark position. Therefore, agents are always aware of the position of the landmark, which is static and placed in the middle of the map.

\paragraph{Action Space} We add two possible actions to the original $6$-dimensional discrete action space, i.e., $\textsc{Pick}$ and $\textsc{Delivery}$. Agents that take the $\textsc{Eat}$ action while carrying food consume the item. Conversely, if an agent attempts to pick up food while already carrying an item, that action is automatically canceled. In general, unsuccessful pickup or delivery attempts do not affect the current state of the environment. 
We remark that items of higher value need the cooperation of multiple agents to be consumed.

\paragraph{Reward} We highlight that the reward structure does not change, as delivering food does not affect the agents' rewards.

\paragraph{Transition Dynamics} The survival and absolute counters start at $T_\sigma = T_\sigma$ and $t = 0$, respectively. When an agent successfully delivers food to the landmark, the counter $T_\sigma$ is increased by the value of the food, incremented by a factor $k_{surv}$. At each step, the survival~(absolute) counter decreases~(increases) by one, and the episode terminates either when $T_\sigma = 0$ or $t = T$. Each agent can carry only one item, while unsuccessful deliveries do not imply negative rewards or changes in the item that is carried. Furthermore, an agent can decide to eat the food that is being carried.

\paragraph{Experimental Setup}
In our experiments, we fix the maximum values for survival and absolute steps as $T_\sigma = 100$ and $T = 500$, respectively. Furthermore, we set the constant for the survival time increment to $k_{sur}=10$: as an example, successfully delivering an item with value $4$ leads to an increment of $T_\sigma$ of $40$ time steps. In our experiments, we consider three different maps of different sizes and amounts of resources. In particular, we consider: $1)$ \textit{LBFwS-Easy}, where the map is a $9\times9$ grid and there are $4$ food items of level $1$ and $4$ food items of level $2$; $2)$ \textit{LBFwS-Medium}, where the map is a $12\times12$ grid and there are $5$ food items of level $1$ and $5$ food items of level $2$; $3)$ \textit{LBFwS-Hard}, where the map is a $15\times15$ grid and there are $6$ food items of level $1$ and $6$ food items of level $2$. We remark that level $2$ items need the cooperation of at least $2$ agents to be picked up. For all the scenarios, the number of agents is set to $10$. Despite having more food available when the difficulty increases, the grid size makes exploration and temporal credit assignment more challenging. Furthermore, this difficulty is compounded by the sparser distribution of food items in the space, which makes successful deliveries to the landmark more complex.

\subsubsection{Hierarchical Graph Generation in Sampling}~\label{app:sampling_hier_graph}
As aforementioned, in the \textit{VMAS Sampling} scenario, we consider a $3$-level dynamic graph $\hiergraph_t$ for the \gls{himppo} model. This graph is built according to the current state $\textbf{S}_t$ of the environment. In particular, the grid is divided into $4$ quadrants, and each worker~(agent) is assigned to a single sub-manager based on its position. Therefore, the set $\gN_s$ of sub-managers is composed of $4$ node indexes, one for each quadrant. Refer to \autoref{fig:build_hier_graph} for a visual example.
We augment the initial representation $\vh_t^{s,0}$ of each sub-manager $s$ with a one-hot vector to give an explicit indexing to the sub-managers. 
However, we remark that this does not prevent the transferability of the architecture to systems with a different number of agents. 
There might be states where a sub-manager has no workers underneath, e.g., the up-right (orange) sub-manager of \autoref{fig:build_hier_graph}. In such cases, the top-level manager sends a goal to account for possible future changes in the dynamics, but the corresponding trajectory is discarded in the learning procedure.

\begin{figure}[h!]
\centering
    \includegraphics[width=0.5\textwidth]{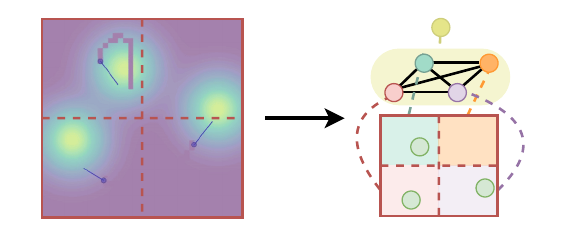}
    \caption{Building the $3$-level hierarchical graph $\hiergraph_t$ at time $t$ for the \textit{VMAS Sampling} scenario.}
    \label{fig:build_hier_graph}
\end{figure}

\subsection{Implementation Details}~\label{app:implementation}
To achieve the optimal configurations, for all the models we performed a grid search on \textit{LBFwS-Medium} environment, focusing mainly on latent dimensions, hidden units, and learning rates. For vector-based observations, encoders are designed as a linear layer. For grid-based observations~(\textit{LBFwS}), we use a Convolutional Network, where we tuned the number of output features and kernel size; then, two linear encoders process the flattened space and time observations to obtain the desired latent dimension. All the models use a discount factor $\gamma = 0.99$ and the Adam~\cite{kingma2014adam} optimizer. For PPO-based models, the entropy coefficient and clipping values are fixed to $0.01$ and $0.2$, respectively. Advantages are approximated using Generalized Advantage Estimator~\cite{schulman2015high} GAE$(\gamma,\lambda)$: $\lambda$ is set to $0.95$, except for upper levels of \gls{himppo}~(i.e., (sub-)managers) and \acrshort{hguppo}, that is fixed to $0$. For the \textit{VMAS Sampling} environment, we tuned the initial action standard deviation, which then decays by $0.05$ every $2.5\cdot10^5$ steps up to $0.1$, except for \acrshort{hguppo} where it decays every $10^6$ steps. Policies are updated for $30$ epochs with a batch size of $128$ every $40$ (\textit{LBFwS} and \textit{SMACv2}) and $32$ (\textit{VMAS Sampling}) episodes. Unless constrained by the action space, we use ReLU as the activation function. For \gls{himppo} and GPPO, the agents' message function is implemented as an MLP with $64$ hidden units in \textit{LBFwS} and \textit{SMACv2}, and as a Graph Convolutional Network~(GCN)~\cite{kipf2017semisupervised} in \textit{VMAS Sampling}. When employing a $3$-level hierarchy (\textit{VMAS Sampling}), \gls{himppo} uses an MLP with $64$ hidden units as the message function for sub-managers. DGN uses the same number of epochs and batch size, but updates the model at each episode and has a buffer of capacity $10^5$. 

\paragraph{\gls{himppo}} All the hyperparameters are shared between the different levels of the hierarchy, except for $\lambda$. We used a learning rate of $0.0001$ and $0.0005$ for the actor and critic, respectively. To keep the number of parameters bounded, all the functions are implemented as linear layers, except for the message function that has $64$ hidden units; we consider $1$ round of message passing. Embedding and goal dimensions are fixed to $64$. For (sub-)managers, encoded features of subordinate nodes are aggregated using the sum. The goal is sampled from a Gaussian policy with an initial standard deviation of $0.5$. For the \textit{VMAS Sampling} environment, where continuous actions are required, we use the same value for the initial standard deviation. For \textit{LBFwS}, we consider a $2$-dimensional convolution with $8$ output channels and kernel size $2$. To ensure stability among different hierarchy levels, we clip gradients to the range $[-1,1]$.

\paragraph{GPPO} This model does not share modules between the actor and critic. The actor and critic have a hidden dimension of $64$ and $32$, respectively. We used a learning rate of $0.0001$ for both. The embedding dimension is fixed to $64$. We perform $2$ message-passing rounds. For the \textit{VMAS Sampling} environment, the initial standard deviation is set to $0.5$. For \textit{LBFwS}, we consider a $2$-dimensional convolution with $32$ output channels and kernel size $2$. Gradients are clipped to the range $[-5,5]$.

\paragraph{H-GUPPO} In this model, the embedding is trained together with the actor, while the critic uses the last representation to get the value function. The actor and critic have a hidden dimension of $64$ and $128$, whereas the learning rates have values $0.0005$ and $0.0001$, respectively. The embedding dimension is fixed to $128$. The depth of the U-Net is fixed to $2$. For the \textit{VMAS Sampling} environment, the initial standard deviation is set to $0.5$. For the \textit{VMAS Sampling} environment, the initial standard deviation is set to $0.8$.  For \textit{LBFwS}, we consider a $2$-dimensional convolution with $32$ output channels and kernel size $2$. Gradients are clipped to the range $[-5,5]$.

\paragraph{MAPPO} In this baseline, the embedding dimension is $128$. The actor has $128$ hidden units, whereas the critic is implemented as a linear layer. We used a learning rate of $0.0005$ and $0.0001$ for the actor and critic, respectively. For the \textit{VMAS Sampling} environment, the initial standard deviation is set to $0.5$. For \textit{LBFwS}, we consider a $2$-dimensional convolution with $8$ output channels and kernel size $2$. Gradients are clipped to the range $[-5,5]$.

\paragraph{IPPO} This model uses an embedding dimension of $64$. The actor and critic have $128$ and $64$ hidden units, whereas the learning rates have values $0.0005$ and $0.00001$, respectively. For the \textit{VMAS Sampling} environment, the initial standard deviation is set to $0.8$. For \textit{LBFwS}, we consider a $2$-dimensional convolution with $8$ output channels and kernel size $2$. Gradients are clipped to the range $[-5,5]$.

\paragraph{DGN} For this baseline, we use a learning rate of $0.0001$ and a latent dimension of $128$. We use a hidden layer of $128$ units and follow the original implementation for the attention model. For \textit{LBFwS}, we consider a $2$-dimensional convolution with $16$ output channels and kernel size $2$. We perform $100$ warm-up episodes, while the target network is updated every $10$ episodes; the exploration coefficient decays from $0.9$ to $0.1$ with a decay rate of $0.0001$ per episode.

\paragraph{MAGIC} In this model, we use a learning rate of $0.001$ and a latent dimension of $256$. The model learns directed communication graphs with self-loops. The two graph attention networks~(sub-processors) have $2$ and $4$ heads, with a hidden size of $32$, and do not use normalization. We learn the first communication graph, while the second one is set as complete. For \textit{LBFwS}, we consider a $2$-dimensional convolution with $32$ output channels and kernel size $2$.

\subsection{Hardware and Software}~\label{app:hardware}
The code for the model and the baselines was developed in \textit{Python 3.10} and made use of \textit{NumPy}~\cite{harris2020array}, \textit{PyTorch}~\cite{10.5555/3454287.3455008}, and \textit{PyTorch Geometric}~\cite{fey2019fast}. We implemented the PPO-based models starting from a public repository~\cite{pytorch_minimal_ppo}~(MIT License), while we adapted the official implementations for both DGN~\footnote{https://github.com/jiechuanjiang/pytorch\_DGN (MIT License)} and MAGIC~\footnote{https://github.com/CORE-Robotics-Lab/MAGIC (MIT License)} to our setting. For the environments, we used the official repositories~\cite{bettini2022vmas,azmani2023cooperative,ellis2023smacv2} and the utility functions of \textit{TorchRL}~\cite{bou2023torchrl}~(MIT License). \textit{VMAS Sampling} uses a GPL3.0 License, while \textit{SMACv2} uses MIT License. Experiment configurations were managed using \textit{Hydra}~\cite{Yadan2019Hydra} and tracked using \textit{Neptune.ai}~\footnote{https://neptune.ai/}. Simulations were run on two workstations equipped with AMD EPYC 7513 and Intel Xeon E5-2650 CPUs. Depending on the model, training times take $2-10$ hours for \textit{VMAS Sampling}, $8-15$ hours for \textit{LBFwS}, and $1-20$ days for \textit{SMACv2}.

\end{document}